\documentclass[10pt,twocolumn,twoside]{IEEEtran}
 \usepackage{pstricks,pst-node,pst-tree,pstricks-add} 
 \usepackage{hyperref}
\usepackage{amsopn}
\usepackage{amsthm}
\usepackage{etex}
\usepackage{spconf,graphicx}
\usepackage{endnotes}
\usepackage{epsfig,psfrag}
\usepackage{pst-all}
\usepackage{amssymb,amsfonts,upref,cite,epsf,color,bm}
\usepackage{graphicx}
\usepackage{color}
\usepackage{amsmath}
\usepackage{graphicx}
\usepackage{calc}
\usepackage{booktabs}
\usepackage{tikz}
\usepackage{pgfplots}
\newcommand\defeq{:=}

\usepackage{algorithm}% http://ctan.org/pkg/algorithms
\usepackage{algpseudocode}% http://ctan.org/pkg/algorithmicx
\floatname{algorithm}{Algorithm}
\algnewcommand\algorithmicinput{\textbf{Input:}}
\algnewcommand\INPUT{\item[\algorithmicinput]}
\algnewcommand\algorithmicoutput{\textbf{Output:}}
\algnewcommand\OUTPUT{\item[\algorithmicoutput]}

\DeclareMathOperator*{\argmin}{arg\;min}

\newcommand\vect[1]{\mathbf #1}

\newcommand{\vd}{\vect{d}}

\newcommand{\vq}{\vect{q}}

\newcommand{\vu}{\vect{u}}  
\newcommand{\vv}{\vect{v}}  

\newcommand{\vx}{\vect{x}}  
\newcommand{\vy}{\vect{y}}  
\newcommand{\vz}{\vect{z}}

\newcommand{\mD}{\mathbf{D}}

\newcommand{\samplesize}{N}

\newcommand\gdop[1]{\mathcal{T}^{(#1)}}
\newcommand\gdopopt{\mathcal{T}}
\newcommand\stackedgdop{\overline{\mathcal{T}}}

\newcommand{\vecdim}{n}

\newcommand{\dataset}{\mathcal{X}}

\newcommand\funclass[2]{\mathcal{S}_{\vecdim}^{#1,#2}}

\newtheorem{theorem}{Theorem}%[section]

\newtheorem{lemma}[theorem]{Lemma}

\usepackage{setspace}

\title{A Fixed-Point of View on Gradient Methods for Big Data}
\name{Alexander Jung}
\address{\normalsize Department of Computer Science, Aalto University, Finland; firstname.lastname(at)aalto.fi
}

\begin{document}
%\thanks{\hspace*{-5mm}The work of ??? was supported by ???.} 
	\maketitle
\begin{abstract}
Interpreting gradient methods as fixed-point iterations, we provide a detailed analysis 
of those methods for minimizing convex objective functions. Due to their conceptual and 
algorithmic simplicity, gradient methods are widely used in machine learning for massive data sets (big data). 
In particular, stochastic gradient methods are considered the de-facto standard for training 
deep neural networks. Studying gradient methods within the realm of fixed-point theory 
provides us with powerful tools to analyze their convergence properties. In particular, gradient methods using inexact 
or noisy gradients, such as stochastic gradient descent, can be studied conveniently using well-known 
results on inexact fixed-point iterations. Moreover, as we demonstrate in this paper, the fixed-point 
approach allows an elegant derivation of accelerations for basic gradient methods. In particular, 
we will show how gradient descent can be accelerated by a fixed-point preserving transformation 
of an operator associated with the objective function. 

\end{abstract}

\begin{keywords} convex optimization, fixed point theory, big data, machine learning, 
contraction mapping, gradient descent, heavy balls
% , subsampling
\end{keywords} 

% no keywords
% \vspace*{-1mm}
%%%%%%%%%%%%%%%%%%%%%%%%%%%%%%%%%%%
\section{Introduction}
 \label{sec_intro}
One of the main recent trends within machine learning and data analytics 
using massive data sets is to leverage the inferential strength of the vast amounts of data 
by using relatively simple, but fast, optimization methods as algorithmic primitives \cite{Bottou2008}. 
Many of these optimization methods are modifications of the basic gradient descent (GD) method. 
Indeed, computationally more heavy approaches, such as interior point methods, 
are often infeasible for a given limited computational budget \cite{Cevher14}. 

Moreover, the rise of deep learning has brought a significant boost for the interest in gradient methods. 
Indeed, a major insight within the theory of deep learning is that for typical high-dimensional models, 
e.g., those represented by deep neural networks, most of the local minima of the cost function 
(e.g., the empirical loss or training error) are reasonably close (in terms of objective value) 
to the global optimum \cite{Goodfellow-et-al-2016}. 
These local minima can be found efficiently by gradient methods such as stochastic gradient descent (SGD), 
which is considered the de-facto standard algorithmic primitive for training 
deep neural networks \cite{Goodfellow-et-al-2016}. 

This paper elaborates on the interpretation of some basic gradient methods such as GD and its 
variants as fixed-point iterations. These fixed-point iterations are obtained for operators associated 
with the convex objective function. Emphasizing the connection to fixed-point theory unleashes some 
powerful tools, e.g., on the acceleration of fixed-point iterations \cite{AndersonAccFixedPoints} 
or inexact fixed-point iterations \cite{InexactFixedPoint,Alfeld82}, for the analysis and construction of 
convex optimization methods. 

In particular, we detail how the convergence of the basic GD iterations can be understood 
from the contraction properties of a specific operator which is associated naturally 
with a differentiable objective function. Moreover, we work out in some detail how the basic GD 
method can be accelerated by modifying the operator underlying GD in a way that preserves its 
fixed-points but decreases the contraction factor which implies faster convergence by 
the contraction mapping theorem. 

{\bf Outline.}
We discuss the basic problem of minimizing convex functions in 
Section \ref{sec_cvx_functions}. We then derive GD, which is a particular first order 
method, as a fixed-point iteration in Section \ref{sec_gradient_descent}. In Section \ref{sec_FOM}, 
we introduce one of the most widely used computational models for convex optimization methods, i.e., 
the model of first order methods. In order to assess the efficiency of GD, which is a particular instance 
of a first order method, we present in Section \ref{sec_lower_bound} a lower bound on the number of 
iterations required by any first order method to reach a given sub-optimality. Using the insight 
provided from the fixed-point interpretation we show how to obtain an accelerated variant of GD 
in Section \ref{sec_AGD}, which turns out to be optimal in terms of convergence rate. 

{\bf Notation.} 
The set of natural numbers is denoted $\mathbb{N} \defeq \{1,2,\ldots \}$. 
Given a vector $\vx=(x_{1},\ldots,x_{\vecdim})^{T} \in \mathbb{C}^{\vecdim}$, we denote its $l$th entry by $x_{l}$. 
The (hermitian) transpose and trace of a square matrix $\mathbf{A} \in \mathbb{C}^{\vecdim \times \vecdim}$ 
are denoted ($\mathbf{A}^{H}$) $\mathbf{A}^{T}$ and ${\rm tr} \{ \mathbf{A} \}$, respectively. The Euclidian norm of a vector $\vx$ is 
denoted $\| \vx \| \defeq \sqrt{ \vx^{H}\vx}$. The spectral norm of a matrix $\mathbf{M}$ 
is denoted $\| \mathbf{M} \| \defeq \max\limits_{\| \vx \|=1} \| \mathbf{M} \vx \|$. The 
spectral decomposition of a positive semidefinite (psd) matrix $\mathbf{Q}\!\in\!\mathbb{C}^{\vecdim \times \vecdim}$ is 
$\mathbf{Q}\!=\!\mathbf{U} {\bf \Lambda} \mathbf{U}^{H}$ with matrix $\mathbf{U}\!=\!\big(\vu^{(1)},\ldots,\vu^{(\vecdim)}\big)$ 
whose columns are the orthonormal eigenvectors $\mathbf{u}^{(i)}\!\in\!\mathbb{C}^{\vecdim}$ of $\mathbf{Q}$ and the 
diagonal matrix ${\bf \Lambda}$ containing the eigenvalues $\lambda_{1}(\mathbf{Q}) \geq \ldots \geq \lambda_{\vecdim}(\mathbf{Q}) \geq 0$. 
For a square matrix $\mathbf{M}$, we denote its spectral radius as $\rho(\mathbf{M}) \defeq \max \{ | \lambda |:  \lambda \mbox{ is an eigenvalue of } \mathbf{M} \}$. 
 
\section{Convex Functions} 
\label{sec_cvx_functions}

A function $f(\cdot): \mathbb{R}^{\vecdim} \rightarrow \mathbb{R}$ is convex if  
\begin{equation}
f((1-\alpha)\vx + \alpha \vy) \leq (1-\alpha) f(\vx) + \alpha f(\vy)  \nonumber
\end{equation}  
holds for any $\vx, \vy \in \mathbb{R}^{\vecdim}$ and $\alpha \in [0,1]$ \cite{Cevher14}. 
For a differentiable function $f(\cdot)$ with gradient 
$\nabla f(\vx)$, a necessary and sufficient condition for convexity 
is \cite[p.\ 70]{BoydConvexBook}
\begin{equation} 
f(\vy) \geq f(\vx)\!+\!(\vy\!-\!\vx)^{T}\nabla f(\vx), \nonumber
\end{equation}
which has to hold for any $\vx,\vy \in \mathbb{R}^{\vecdim}$. 

Our main object of interest in this paper is the optimization problem 
\begin{equation}
\label{equ_opt_problem}
\vx_{0} \in \argmin_{\vx \in \mathbb{R}^{\vecdim}} f(\vx). 
\end{equation} 
Given a convex function $f(\vx)$, we aim at finding a point $\vx_{0}$ 
with lowest function value $f(\vx_{0})$, i.e., $f(\vx_{0}) = \min_{\vx} f(\vx)$. 

%{\bf Linear Regression.} 
In order to motivate our interest in optimization problems like \eqref{equ_opt_problem}, 
consider a machine learning problem based on training data $\dataset \defeq\{\vz^{(i)}\}_{i=1}^{\samplesize}$ 
consisting of $\samplesize$ data points $\vz^{(i)}\!=\!(\vd^{(i)},y^{(i)})$ with feature vector $\vd^{(i)} \in \mathbb{R}^{\vecdim}$ 
(which might represent the RGB pixel values of a webcam snapshot) and output or label $y^{(i)} \in \mathbb{R}$ (which might represent the local temperature 
during the snapshot). We wish to predict the label $y^{(i)}$ by a linear combination of the features, i.e.,
\begin{equation}
\label{equ_def_linear_predictor}
y^{(i)} \approx \vx^{T} \vd^{(i)}. 
\end{equation}
The choice for the weight vector $\vx\in\mathbb{R}^{\vecdim}$ is typically based on balancing the 
empirical risk incurred by the predictor \eqref{equ_def_linear_predictor}, i.e., 
\vspace*{-2mm}
\begin{equation} 
(1/\samplesize) \sum_{i=1}^{\samplesize} (y^{(i)}\!-\!\vx^{T} \vd^{(i)})^2, \nonumber
\end{equation} 
with some regularization term, e.g., measured by the squared norm $\| \vx \|^{2}$. Thus, 
the learning problem amounts to solving the optimization problem 
\begin{equation}
\label{equ_lrproblem}
\vx_{0}\!=\!\argmin_{\vx \in \mathbb{R}^{\vecdim}} (1/\samplesize) \sum_{i=1}^{\samplesize} (y^{(i)}\!-\!\vx^{T} \vd^{(i)})^2\!+\!\lambda \| \vx \|^{2}.
\end{equation}  
The learning problem \eqref{equ_lrproblem} is precisely of the form \eqref{equ_opt_problem} with the convex objective function 
\begin{equation}
\label{equ_objective_linreg}
f(\vx) \defeq (1/\samplesize) \sum_{i=1}^{\samplesize} (y^{(i)}\!-\!\vx^{T} \vd^{(i)})^2\!+\!\lambda \| \vx \|^{2}.
\end{equation} 
By choosing a large value for the regularization parameter $\lambda$, we de-emphasize the relevance of the training error and thus 
avoid overfitting. However, choosing $\lambda$ too large induces a bias if the true underlying 
weight vector has a large norm \cite{Goodfellow-et-al-2016,BishopBook}. A principled approach 
to find a suitable value of $\lambda$ is cross validation \cite{Goodfellow-et-al-2016,BishopBook}.

{\bf Differentiable Convex Functions.} 
Any differentiable function $f(\cdot)$ is accompanied by its gradient operator 
\begin{equation}
\label{equ_def_gradient_operator}
\nabla f: \mathbb{R}^{\vecdim} \rightarrow \mathbb{R}^{\vecdim}, \vx \mapsto \nabla f(\vx). 
\end{equation}
While the gradient operator $\nabla f$ is defined for any (even non-convex) differentiable 
function, the gradient operator of a convex function satisfies a strong structural 
property, i.e., it is a monotone operator \cite{BauschkeCombettesBook}. 

{\bf Smooth and Strongly Convex Functions.} If all second order partial derivatives of 
the function $f(\cdot)$ exist and are continuous, then $f(\cdot)$ is convex if and only if \cite[p. 71]{BoydConvexBook}
\begin{equation}
\nabla^{2} f(\vx) \succeq \mathbf{0} \mbox{ for every } \vx \in \mathbb{R}^{\vecdim}.  \nonumber
\end{equation}
We will focus on a particular class of twice differentiable 
convex functions, i.e., those with Hessian $\nabla^{2} f(\vx)$ satisfying 
\begin{equation}
\label{double_bound_hessian}
L\!\leq\!\lambda_{l} \big( \nabla^{2} f(\vx) \big)\!\leq\!U \mbox{ for every } \vx \in \mathbb{R}^{\vecdim},
\end{equation} 
with some known constants $U \geq L>0$.

The set of convex functions $f(\cdot):\mathbb{R}^{\vecdim} \rightarrow \mathbb{R}$ 
satisfying \eqref{double_bound_hessian} will be denoted $\funclass{L}{U}$. 
As it turns out, the difficulty of finding the minimum of some function $f(\cdot) \in \funclass{L}{U}$ 
using gradient methods is essentially governed by the 
\begin{equation}
\label{equ_def_condition_number}
\hspace*{-2mm}\mbox{ condition number } \kappa \defeq U/L \mbox{ of the function class } \funclass{L}{U}. 
\end{equation}
Thus, regarding the difficulty of optimizing the functions $f(\cdot) \in \funclass{L}{U}$, 
the absolute values of the bounds $L$ and $U$ in \eqref{double_bound_hessian} 
are not crucial, only their ratio $\kappa=U/L$ is. 

One particular sub-class of functions $f(\cdot) \in \funclass{L}{U}$ , which is of paramount 
importance for the analysis of gradient methods, are quadratic functions of the form
\begin{equation}
\label{equ_quadratic_function}
f(\vx) = (1/2) \vx^{T} \mathbf{Q} \vx \!+\!\mathbf{q}^{T} \vx\!+\!c, 
\end{equation}
with some vector $\vq \in \mathbb{R}^{\vecdim}$ and a psd matrix $\mathbf{Q} \in \mathbb{R}^{\vecdim \times \vecdim}$ having eigenvalues 
$\lambda(\mathbf{Q}) \in [L,U]$. As can be verified easily, the gradient and Hessian of a quadratic function of the 
form \eqref{equ_quadratic_function} are obtained as $\nabla f(\vx)\!=\!\mathbf{Q} \vx\!+\!\mathbf{q}$ and 
$\nabla^{2} f(\vx)\!=\!\mathbf{Q}$, respectively. 

It turns out that most of the results (see below) on gradient methods for minimizing quadratic functions of the form 
\eqref{equ_quadratic_function}, with some matrix $\mathbf{Q}$ having eigenvalues $\lambda(\mathbf{Q}) \in [L,U]$, 
apply (with minor modifications) also when expanding their scope from quadratic functions to the larger set $\funclass{L}{U}$. 
This should not come as a surprise, since any function $f(\cdot)\!\in\!\funclass{L}{U}$ can be approximated 
locally around a point $\vx_{0}$ by a quadratic function which is obtained by a truncated 
Taylor series \cite{RudinBookPrinciplesMatheAnalysis}. 
In particular, we have \cite[Theorem 5.15]{RudinBookPrinciplesMatheAnalysis}
\begin{align}
\label{equ_truncated_taylor_funclass}
f(\vx) &\!=\! f(\vx_{0})\!+\!(\vx\!-\!\vx_{0})^{T} \nabla f(\vx_{0}) \nonumber \\ 
& + (1/2)  (\vx\!-\!\vx_{0})^{T}  \nabla^{2} f(\vu) (\vx\!-\!\vx_{0}),
\end{align}
where $\vu = \eta \vx\!+\!(1\!-\!\eta) \vx_{0}$ with some $\eta \in [0,1]$. 

The crucial difference between the quadratic function \eqref{equ_quadratic_function} and a 
general function $f(\cdot) \in \funclass{L}{U}$ is that the matrix $\nabla^{2} f(\vz)$ appearing in 
the quadratic form in \eqref{equ_truncated_taylor_funclass} typically varies 
with the point $\vx$. In particular, we can rewrite \eqref{equ_truncated_taylor_funclass} as   
\begin{align}
\label{equ_approx_funclass_quadratics}
f(\vx) &\!=\! \nonumber \\ 
&\hspace*{-8mm} f(\vx_{0})\!+\!(\vx\!-\!\vx_{0})^{T} \nabla f(\vx_{0})\!+\!(1/2)  (\vx\!-\!\vx_{0})^{T} \mathbf{Q} (\vx\!-\!\vx_{0}) \nonumber \\ 
&\hspace*{-8mm} +  (1/2)  (\vx\!-\!\vx_{0})^{T}  (\nabla^{2} f(\vz)\!-\!\mathbf{Q})(\vx\!-\!\vx_{0}),
\end{align} 
with $\mathbf{Q}=\nabla^{2}f(\vx_{0})$. The last summand in \eqref{equ_approx_funclass_quadratics} 
quantifies the approximation error 
\begin{align} 
\label{equ_approx_error_quadratic}
\varepsilon(\vx) & \defeq f(\vx) - \tilde{f}(\vx)  \\ \nonumber 
& =   (1/2)(\vx\!-\!\vx_{0})^{T}  (\nabla^{2} f(\vz)\!-\!\mathbf{Q})(\vx\!-\!\vx_{0}) 
\end{align} 
obtained when approximating a function $f(\cdot)\in \funclass{L}{U}$ 
with the quadratic $\tilde{f}(\vx)$ obtained from \eqref{equ_quadratic_function} with the choices 
\begin{align}
\mathbf{Q}&\!=\!\nabla^{2}f(\vx_{0}), \nonumber \\ 
 \vq& \!=\! \nabla f(\vx_{0})-\mathbf{Q} \vx_{0} \mbox{ and } \nonumber \\ 
 c&\!=\!f(\vx_{0})\!+\!(1/2)\vx_{0}^{T} \mathbf{Q} \mathbf{x}_{0}\!-\!\vx_{0}^{T} \nabla f(\vx_{0}).  \nonumber
\end{align}
According to \eqref{double_bound_hessian}, which implies 
$\big\| \nabla^{2} f(\vx_{0}) \big\| , \big\| \nabla^{2} f(\vz)\big\|  \leq U$, we can bound the approximation error \eqref{equ_approx_error_quadratic} as 
\begin{align}
 \varepsilon(\vx) & \!\leq\! U \| \vx\!-\!\vx_{0} \|^{2}.  \nonumber
\end{align} 
Thus, we can ensure a arbitrarily small approximation error $\varepsilon$ by considering $f(\cdot)$ only over a 
neighbourhood $\mathcal{B}(\vx_{0},r)\defeq \{\vx:  \| \vx\!-\!\vx_{0} \| \leq r \}$ with sufficiently small radius $r>0$. 

Let us now verify that learning a (regularized) linear regression model (cf.\ \eqref{equ_lrproblem}) 
amounts to minimizing a convex quadratic function of the form \eqref{equ_quadratic_function}. 
Indeed, using some elementary linear algebraic manipulations, we can 
rewrite the objective function in \eqref{equ_objective_linreg} as a quadratic of the form \eqref{equ_quadratic_function} 
using the particular choices $\mathbf{Q}\!=\!\mathbf{Q}_{\rm LR}$ and $\vq\!=\!\vq_{\rm LR}$ with 
\begin{equation} 
\label{equ_quadratic_LR}
\mathbf{Q}_{\rm LR} \!\defeq\! \lambda \mathbf{I}\!+\!\frac{1}{\samplesize}\sum_{i=1}^{\samplesize} \vd^{(i)} \big( \vd^{(i)} \big)^{T} 
\mbox{, and } \vq_{\rm LR} \!\defeq\! \frac{2}{\samplesize} \sum_{i=1}^{\samplesize} y^{(i)} \vd^{(i)}.
\end{equation} 
The eigenvalues of the matrix $\mathbf{Q}_{\rm LR}$ obey \cite{golub96} 
\begin{equation}
\lambda \leq \lambda_{l}\big( \mathbf{Q}_{\rm LR} \big) \leq \lambda + \lambda_{1}(\mD^{T} \mD)  \nonumber
\end{equation}
with the data matrix $\mD\defeq \big( \vd^{(1)},\ldots,\vd^{(\samplesize)} \big) \in \mathbb{R}^{\vecdim \times \samplesize}$. 
Hence, learning a regularized linear regression model via \eqref{equ_lrproblem} amounts to minimizing a convex quadratic function 
$f(\cdot) \in \funclass{L}{U}$ with $L\!=\!\lambda$ and $U\!=\!\lambda\!+\!\lambda_{1}(\mD^{T} \mD)$, where $\lambda$ denotes 
the regularization parameter used in \eqref{equ_lrproblem}. 

\section{Gradient Descent}
\label{sec_gradient_descent}

Let us now show how one of the most basic methods for solving the problem \eqref{equ_opt_problem}, i.e., the GD method, 
can be obtained naturally as fixed-point iterations involving the gradient operator $\nabla f$ (cf.\ \eqref{equ_def_gradient_operator}). 

Our point of departure is the necessary and sufficient condition \cite{BoydConvexBook}
\begin{equation}
\label{equ_zero_gradient}
\nabla f(\vx_{0}) = \mathbf{0},  
\end{equation} 
for a vector $\vx_{0} \in \mathbb{R}^{\vecdim}$ to be optimal for the problem \eqref{equ_opt_problem} with 
a convex differentiable objective function $f(\cdot) \in \funclass{L}{U}$. 
%%%%%%%%%%%%%%%
\begin{lemma}
\label{lem_fixed_points}
We have $\nabla f(\vx) = \mathbf{0}$ if and 
only if the vector $\vx \in \mathbb{R}^{\vecdim}$ is a fixed point of the operator 
\begin{equation}
\label{equ_def_operator_alpha}
\gdop{\alpha}: \mathbb{R}^{\vecdim} \rightarrow \mathbb{R}^{\vecdim}: \vx \mapsto \vx - \alpha \nabla f(\vx),
\end{equation} 
for an arbitrary but fixed non-zero $\alpha \in \mathbb{R} \setminus \{0\}$. 
Thus, 
\begin{equation} 
\nabla f(\vx)=\mathbf{0} \mbox{ if and only if } \gdop{\alpha} \vx = \vx. \nonumber
\end{equation} 
\end{lemma}
%%%%%%%%%%%%%%%%%% 
\begin{proof}
Consider a vector $\vx$ such that $\nabla f (\vx) = \mathbf{0}$. Then, 
\begin{equation}
\gdop{\alpha} \vx  \stackrel{\eqref{equ_def_operator_alpha}}{=} \vx - \alpha \nabla f(\vx) = \vx.    \nonumber
\end{equation} 

Conversely, let $\vx$ be a fixed point of $\gdop{\alpha}$, i.e., 
\begin{equation} 
\label{fixed_point_proof}
\gdop{\alpha} \vx = \vx. 
\end{equation}
Then, 
\begin{align}
\nabla f(\vx) & \stackrel{\alpha\!\neq\!0}{=} (1/\alpha) (\vx - (\vx - \alpha \nabla f(\vx))) \nonumber \\ 
  & \stackrel{\eqref{equ_def_operator_alpha}}{=}  (1/\alpha) (\vx - \gdop{\alpha} \vx) \nonumber \\ 
  &  \stackrel{\eqref{fixed_point_proof}}{=} \mathbf{0}.   \nonumber
\end{align} 
\end{proof} 

According to Lemma \ref{lem_fixed_points}, the solution $\vx_{0}$ of the optimization problem 
\eqref{equ_opt_problem} is obtained as the fixed point of the operator $\gdop{\alpha}$ (cf.\ \eqref{equ_def_operator_alpha}) with 
some non-zero $\alpha$. As we will see shortly, the freedom in choosing different values 
for $\alpha$ can be exploited in order to compute the fixed points of $\gdop{\alpha}$ more efficiently. 

A straightforward approach to finding the fixed-points of an 
operator $\gdop{\alpha}$ is via the fixed-point iteration
\begin{equation}
\label{equ_fixed_point_iterations}
\vx^{(k+1)} = \gdop{\alpha} \vx^{(k)}.
\end{equation} 
By tailoring a fundamental result of analysis (cf. \cite[Theorem 9.23]{RudinBookPrinciplesMatheAnalysis}), 
we can characterize the convergence of the sequence $\vx^{(k)}$ obtained from \eqref{equ_fixed_point_iterations}. 
%%%%%%%%%
\begin{lemma}
\label{lem_contraction_mapping}
Assume that for some $q\!\in\![0,1)$, we have
\begin{equation}
\label{equ_contraction_inqu}
 \big\| \gdop{\alpha} \vx - \gdop{\alpha} \vy\big\| \leq q \| \vx - \vy \|,
\end{equation} 
for any $\vx, \vy \in \mathbb{R}^{\vecdim}$.
Then, the operator $\gdop{\alpha}$ has a unique fixed point $\vx_{0}$ 
and the iterates $\vx^{(k)}$ (cf.\ \eqref{equ_fixed_point_iterations}) satisfy 
\begin{equation}
\label{equ_iteration_error}
\| \vx^{(k)} - \vx_{0} \| \leq \| \vx^{(0)} - \vx_{0} \| q^{k} . 
\end{equation} 
\end{lemma}
%%%%%%%%%%%%
\begin{proof}
Let us first verify that the operator $\gdop{\alpha}$ cannot have 
two different fixed points. Indeed, assume there would be two different fixed points $\vx$, $\vy$ such that 
\begin{equation} 
\label{equ_fixed_points_x_y}
\vx = \gdop{\alpha} \vx \mbox{, and }\vy = \gdop{\alpha} \vy.
\end{equation} 
This would imply, in turn, 
\begin{align}
%\label{equ_condition_contraction}
q \| \vx - \vy \| & \stackrel{\eqref{equ_contraction_inqu}}{\geq}  \big\| \gdop{\alpha} \vx - \gdop{\alpha} \vy \big\|  \nonumber  \\ 
&  \stackrel{\eqref{equ_fixed_points_x_y}}{=}  \| \vx - \vy \|.  \nonumber
\end{align} 
However, since $q <1$, this inequality can only be satisfied if $\| \vx - \vy \| = 0$, 
i.e., we must have $\vx = \vy$. Thus, we have shown that no two different fixed points 
can exist. The existence of one unique fixed point $\vx_{0}$ follows from \cite[Theorem 9.23]{RudinBookPrinciplesMatheAnalysis}. 

The estimate \eqref{equ_iteration_error} can be obtained by induction and noting 
\begin{align}
\| \vx^{(k+1)} - \vx_{0} \| & \stackrel{\eqref{equ_fixed_point_iterations}}{=} \| \gdop{\alpha} \vx^{(k)} - \vx_{0} \| \nonumber \\ 
&  \stackrel{(a)}{=} \| \gdop{\alpha} \vx^{(k)} - \gdop{\alpha}  \vx_{0} \| \nonumber \\ 
&  \stackrel{\eqref{equ_contraction_inqu}}{\leq} q \| \vx^{(k)} -  \vx_{0} \|.  \nonumber
\end{align} 
Here, step $(a)$ is valid since $\vx_{0}$ is a fixed point of $\gdop{\alpha}$, i.e., $\vx_{0} = \gdop{\alpha} \vx_{0}$.
\end{proof}

In order to apply Lemma \ref{lem_contraction_mapping} to \eqref{equ_fixed_point_iterations}, 
we have to ensure that the operator $\gdop{\alpha}$ is a contraction, i.e., it satisfies \eqref{equ_contraction_inqu} 
with some contraction coefficient $q \in [0,1)$. For the operator $\gdop{\alpha}$ (cf.\ \eqref{equ_def_operator_alpha}) associated 
with the function $f(\cdot) \in \funclass{L}{U}$ this can be verified by standard results from vector analysis. 
\begin{lemma}
\label{lemma_condition_contraction}
Consider the operator $\gdop{\alpha}: \vx \mapsto \vx - \alpha \nabla f (\vx)$ with some convex function $f(\cdot) \in \funclass{L}{U}$. 
Then,  
\begin{equation}
 \big\| \gdop{\alpha} \vx - \gdop{\alpha} \vy\big\| \leq q(\alpha) \| \vx - \vy \| \nonumber
\end{equation} 
with contraction factor 
\begin{equation}
\label{equ_def_contraction_factor}
q(\alpha) \defeq \max\{ |1\!-\!U \alpha |, |1\!-\!L \alpha| \}.
\end{equation}
\end{lemma} 
\begin{proof}
First, 
\begin{align} 
\label{equ_equ_contraction_gradient_111}
\gdop{\alpha} \vx - \gdop{\alpha} \vy & \stackrel{\eqref{equ_def_operator_alpha}}{=} (\vx\!-\!\vy)\!-\!\alpha(\nabla f(\vx)\!-\!\nabla f(\vy)) \nonumber \\ 
& \stackrel{(a)}{=}  (\vx\!-\!\vy)\!-\!\alpha\nabla^{2} f(\vz)  (\vx\!-\!\vy) \nonumber \\ 
& = (\mathbf{I}\!-\!\alpha \nabla^{2} f(\vz)  )(\vx\!-\!\vy)  
\end{align} 
using $\vz=\eta \vx + (1-\eta) \vy$ with some $\eta \in [0,1]$. Here, we used in step $(a)$ the mean value theorem 
of vector calculus \cite[Theorem 5.10]{RudinBookPrinciplesMatheAnalysis}. 

Combining \eqref{equ_equ_contraction_gradient_111} with the submultiplicativity of 
Euclidean and spectral norm \cite[p.\ 55]{golub96} yields 
\begin{equation}
\label{equ_proof_GD_contrac_submult}
\| \gdop{\alpha} \vx - \gdop{\alpha} \vy \| \leq \| \vx - \vy \| \| \mathbf{I} - \alpha \nabla^{2} f(\vz) \|. 
\end{equation} 
The matrix $\mathbf{M}^{(\alpha)} \!\defeq\!\mathbf{I}\!-\!\alpha \nabla^{2} f(\vz)$ is symmetric 
($\mathbf{M}^{(\alpha)} = \big(\mathbf{M}^{(\alpha)}\big)^{T}$) with real-valued eigenvalues \cite{golub96}
\begin{equation}
\label{equ_lambda_l_interval}
\lambda_{l}\big( \mathbf{M}^{(\alpha)} \big) \in [1-U \alpha, 1- L \alpha]. 
\end{equation}
Since also 
\begin{align}
\label{equ_upper_bound_specnorm_LU} 
\| \mathbf{M}^{(\alpha)}  \| & = \max\{ | \lambda_{l}| \} \nonumber \\ 
  & \stackrel{\eqref{equ_lambda_l_interval}}{\leq} \max\{ |1\!-\!U \alpha |, |1\!-\!L \alpha| \},
\end{align} 
we obtain from \eqref{equ_proof_GD_contrac_submult}  
\begin{equation} 
\| \gdop{\alpha} \vx\!-\!\gdop{\alpha} \vy \|\!\stackrel{\eqref{equ_upper_bound_specnorm_LU}}{\leq}\! \| \vx\!-\!\vy \| \max\{ |1\!-\!U \alpha |, |1\!-\!L \alpha| \}.  \nonumber
\end{equation}
%By the mean value theorem of calculus, we have
%\begin{equation}
%\vy - \alpha \nabla f(\vy) = \nabla f(\vx) + (\vy
%\end{equation}
\end{proof}
It will be handy to write out the straightforward combination of Lemma \ref{lem_contraction_mapping} 
and Lemma \ref{lemma_condition_contraction}. 
\begin{lemma}
\label{lem_main_charac_contract}
Consider a convex function $f(\cdot) \in \funclass{L}{U}$ with the unique minimizer $\vx_{0}$, i.e., $f(\vx_{0}) = \min_{\vx} f(\vx)$. 
We then construct  the operator $\gdop{\alpha}: \vx \mapsto \vx - \alpha \nabla f (\vx)$ 
with a step size $\alpha$ such that 
\begin{equation} 
q(\alpha) \stackrel{\eqref{equ_def_contraction_factor}}{=} \max\{ |1\!-\!U \alpha |, |1\!-\!L \alpha| \} < 1. \nonumber
\end{equation} 
Then, starting from an arbitrary initial guess $\vx^{(0)}$, the iterates 
$\vx^{(k)}$ (cf.\ \eqref{equ_fixed_point_iterations}) satisfy 
\begin{equation}
\label{equ_iteration_error_1}
\| \vx^{(k)} - \vx_{0} \| \leq \| \vx^{(0)} - \vx_{0} \| \big[q(\alpha)\big]^{k} .
\end{equation} 
%with $q(\alpha)= \max\{ |1\!-\!U \alpha |, |1\!-\!L \alpha| \}$. 
\end{lemma}

According to Lemma \ref{lem_main_charac_contract}, and also illustrated in Figure \ref{fig_fixed_point}, 
starting from an arbitrary initial guess 
$\vx^{(0)}$, the sequence $\vx^{(k)}$ generated by the fixed-point iteration \eqref{equ_fixed_point_iterations} 
is guaranteed to converge to the unique solution $\vx_{0}$ of \eqref{equ_opt_problem}, i.e., $\lim_{k \rightarrow \infty} \vx^{(k)} = \vx_{0}$. 
What is more, this convergence is quite fast, since the error $\| \vx^{(k)}\!-\!\vx_{0} \|$ decays 
at least exponentially according to \eqref{equ_iteration_error_1}. Loosely speaking, this exponential decrease implies that 
the number of additional iterations required to have on more correct digit in $\vx^{(k)}$ is 
constant. 

Let us now work out the iterations \eqref{equ_fixed_point_iterations} more explicitly by inserting 
the expression \eqref{equ_def_operator_alpha} for the operator $\gdop{\alpha}$. 
We then obtain the following equivalent representation of \eqref{equ_fixed_point_iterations}: 
\begin{equation}
\label{equ_iteration_GD}
\vx^{(k+1)} =\vx^{(k)}- \alpha \nabla f(\vx^{(k)}).
\end{equation}
This iteration is nothing but plain vanilla GD using a fixed step size $\alpha$ \cite{Goodfellow-et-al-2016}. 

Since the GD iteration \eqref{equ_iteration_GD} is precisely the fixed-point iteration \eqref{equ_fixed_point_iterations}, we can 
use Lemma \ref{lem_main_charac_contract} to characterize the convergence (rate) of GD. 
In particular, convergence of GD is ensured by choosing the step size of GD \eqref{equ_iteration_GD} such that 
$q(\alpha) =\max\{ |1\!-\!U \alpha |, |1\!-\!L \alpha| \} < 1$. Moreover, in order to make the convergence 
as fast as possible we need to chose the step size $\alpha=\alpha^{*}$ which makes 
the contraction factor $q(\alpha)$ (cf.\ \eqref{equ_def_contraction_factor}) as small as possible. 

\begin{figure}[ht]
\begin{pspicture}[algebraic](-5mm,-1cm)(7,7)
\psaxes[labels=none,ticks=none]{->}(7.4,7.4)
\psplot[linecolor=red,linewidth=3pt]{0}{7}{5/20+8*x/10}
\psline[linewidth=0.5pt](7,7)
\rput[tl](5,4.1){$\gdop{\alpha}$}
\rput[tl](5.4,6.4){$\vx\!=\!\vy$}
\rput[tl](6.6,-0.2){$\vx^{(0)}$}
 \rput[tl](7.2,0.3){$\vx$}
 \rput[tl](-0.4,7.2){$\vy$}
 \rput[tl](5.5,-0.2){$\vx^{(1)}$}
  \rput[tl](-0.6,5.3){$\vx^{(2)}$}
   \rput[tl](1.2,-0.2){$\vx_{0}$}
\psline[linewidth=0.5pt,linestyle=dashed](0,4.95)(5,4.95)
\psline[linewidth=0.5pt](1.2,1.2)(1.2,0)
\psline[linewidth=0.5pt,linestyle=dashed](5.85,0)(5.85,4.95)
  \psFixpoint[linecolor=blue,linestyle=dashed]{7}{5/20+8*x/10}{20}
 \vspace*{-4mm}
\end{pspicture}
 \vspace*{-3mm}
\caption{Fixed-point iterations for a contractive mapping $\gdop{\alpha}$ with the unique fixed point $\vx_{0}$.}
\label{fig_fixed_point}
 \vspace*{3mm}
\end{figure}
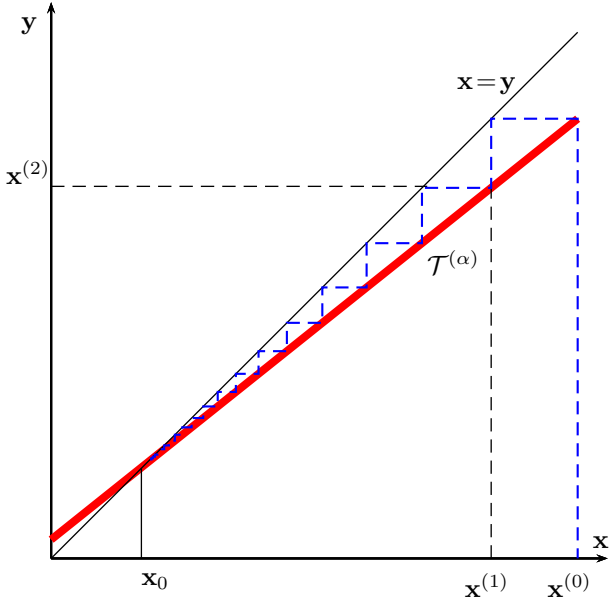

\vspace*{14mm}
\begin{figure}[htb]
\vspace*{14mm}
\begin{center}
\hspace*{-20mm}
\begin{pspicture}(-0.4,-0.4)(4.3,2.4)
\psset{unit=1.5cm}
\psaxes[labels=none]{->}(0,0)(-0.4,-0.4)(4.1,2.4)
\psline[linewidth=1pt]{}(0,2)(2,0)
\psline[linewidth=1pt]{}(2,0)(4,2)
\psline[linewidth=1pt,linestyle=dashed]{}(0,2)(4,1)
\psline[linewidth=1pt,linestyle=dotted]{}(0,1.2)(3.2,1.2)
\rput[tl](-0.9,1.4){$q^{*}\!=\!\frac{\kappa-1}{\kappa+1}$}
\psline[linewidth=1pt,linestyle=dotted]{}(3.2,1.2)(3.2,0)
\psline[linewidth=1pt,linecolor=red]{}(0,2.015)(3.2,1.215)
\psline[linewidth=1pt,linecolor=red]{}(3.2,1.215)(4,2.015)
\rput[tl](3.0,-0.05){$\alpha^{*}\!=\!\frac{2}{L+U}$}
\rput[tl](2.3,1.7){$q(\alpha)$}
\rput[tl](-0.4,2){$1$}
\rput[tl](2,-0.2){$1/U$}
\rput[tl](4.2,0.1){$\alpha$}
\rput[tl](4,2){$|1\!-\!U \alpha|$}
\rput[tl](4,1){$|1\!-\!L \alpha|$}
\end{pspicture}
\end{center}
\caption{Dependence of contraction factor $q(\alpha)= \max \{ |1\!-\!\alpha L|,|1\!-\!\alpha U|\}$ on step size $\alpha$.}
\label{fig_contrac_alpha_func}
\end{figure}
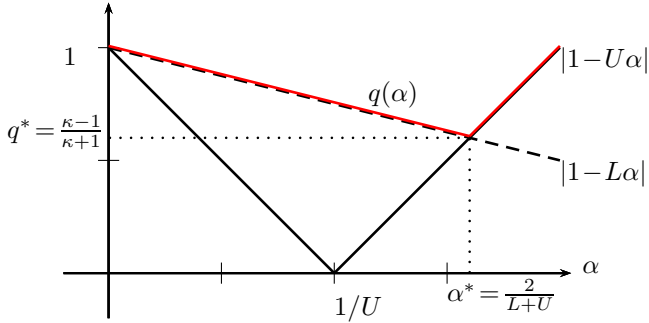

In Figure \ref{fig_contrac_alpha_func}, we illustrate how the quantifies $|1 - \alpha L|$ and $|1 - \alpha U|$ evolve  
as the step size $\alpha$ (cf.\ \eqref{equ_iteration_GD}) is varied. From Figure \ref{fig_contrac_alpha_func} we can easily read 
off the optimal choice 
\begin{equation} 
\label{equ_opt_step_size}
\alpha^{*}= \frac{2}{L + U} 
\end{equation} 
yielding the smallest possible contraction factor 
\begin{equation} 
q^{*} = \min_{\alpha \in [0,1]} q(\alpha) = \frac{U\!-\!L}{U\!+\!L} \stackrel{\eqref{equ_def_contraction_factor}}{=} \frac{\kappa\!-\!1}{\kappa\!+\!1}. \nonumber
\end{equation} 
We have arrived at the following characterization of 
GD for minimizing convex functions $f(\cdot) \in \funclass{L}{U}$. 
\begin{theorem}
\label{equ_theorem_GD_convergence}
Consider the optimization problem \eqref{equ_opt_problem} with objective function $f(\cdot) \in \funclass{L}{U}$, where 
the parameters $L$ and $U$ are fixed and known. Starting from an arbitrarily chosen initial 
guess $\vx^{(0)}$, we construct a sequence by GD \eqref{equ_iteration_GD} 
using the optimal step size \eqref{equ_opt_step_size}. Then,  
\begin{equation}
\label{equ_upper_bound_GD}
\| \vx^{(k)}\!-\!\vx_{0} \| \!\leq\! \bigg(\frac{\kappa\!-\!1}{\kappa\!+\!1}\bigg)^{k}  \| \vx^{(0)}\!-\!\vx_{0} \|. 
\end{equation}
\end{theorem}
In what follows, we will use the shorthand $\gdopopt \defeq \gdop{\alpha^{*}}$ for the 
gradient operator $\gdop{\alpha}$ (cf.\ \eqref{equ_def_operator_alpha}) obtained for the optimal step size
 $\alpha = \alpha^{*}$ (cf.\ \eqref{equ_opt_step_size}). 

\section{First Order Methods}
\label{sec_FOM} 

Without a computational model taking into account a finite amount of resources, the study of the 
computational complexity inherent to \eqref{equ_opt_problem} becomes meaningless. 
Consider having unlimited computational resources at our disposal. Then, we could build an 
``optimization device'' which maps each function $f(\cdot) \in \funclass{L}{U}$ to its unique minimum $\vx_{0}$. 
Obviously, this approach is infeasible since we cannot perfectly 
represent such a mapping, let alone its domain $\funclass{L}{U}$, using a physical hardware  
which allows us only to handle finite sets instead of continuous spaces like $\funclass{L}{U}$. 

Let us further illustrate the usefulness of using a computational model in the context of machine 
learning from massive data sets (big data). In particular, as we have seen in the previous section, 
the regularized linear regression model \eqref{equ_lrproblem} amounts 
to minimizing a convex quadratic function \eqref{equ_quadratic_function} with the particular choices \eqref{equ_quadratic_LR}. 
Even for this most simple machine learning model, it is typically infeasible to have access to a 
complete description of the objective function \eqref{equ_quadratic_function}. 

Indeed, in order to fully specify the quadratic function in \eqref{equ_quadratic_function}, we need to fully specify the matrix 
$\mathbf{Q}\in \mathbb{R}^{\vecdim \times \vecdim}$ and the vector $\vq \in \mathbb{R}^{\vecdim}$. 
For the (regularized) linear regression model \eqref{equ_lrproblem} this would require to 
compute $\mathbf{Q}_{\rm LR}$ (cf.\ \eqref{equ_quadratic_LR}) from the training data $\dataset = \{ \vz^{(i)} \}_{i=1}^{\samplesize}$. 
Computing the matrix $\mathbf{Q}_{\rm LR}$ in a naive way, i.e., without exploiting any additional structure, amounts to 
a number of arithmetic operations on the order of $\samplesize \cdot \vecdim^{2}$. This might be prohibitive in a typical big data application 
with $\samplesize$ and $\vecdim$ being on the order of billions and using distributed 
storage of the training data $\dataset$ \cite{EusipcoTutBigDat}. 
\begin{figure}
	\begin{center}
		\hspace*{-25mm}
		\begin{pspicture}(-0.4,-0.4)(4.3,2.4)
		\psset{unit=1cm}
		\psline[linewidth=1pt]{->}(0,1.5)(0.5,0)
		\psline[linewidth=1pt]{->}(1,0)(1.5,1.7)
		\psline[linewidth=1pt]{->}(3,1.5)(3.5,0)
		\psline[linewidth=1pt]{->}(4,0)(4.5,1.7)
		\psline[linewidth=0.5pt]{}(-0.3, 0.7)(6,0.7)
   		\rput[tl](-0.4,3.5){ $\min_{\vx} f(\vx)\!\defeq\!\frac{1}{N} \sum_{i=1}^{N} (y^{(i)}\!-\!\vx^{T} \vd^{(i)})^{2}\!+\!\lambda \| \vx \|^{2}$}
	        \rput[tl](-0.4,2.8){FOM $\vx^{(k+1)} \in {\rm span}\{\vx^{(0)},\nabla f(\vx^{(0)}),\ldots,\nabla f(\vx^{(k)}) \}$}
		\rput[tl](-0.4,2.2){$\vx^{(0)}$}
		\rput[tl](1,2.2){$\nabla f(\vx^{(0)})$}
		\rput[tl](2.7,2.2){$\vx^{(1)}$}
		\rput[lt](4.7,1.6){\parbox[c]{3cm}{``application'' layer (e.g., Python routine)}}
		\rput[lt](4.7,0.5){\parbox[c]{3cm}{``data'' layer (e.g., Hadoop)}}
		\rput[tl](3.7,2.2){$\nabla f(\vx^{(1)})$}
	%	\rput(2.5,-0.5){\psovalbox[linecolor = black,linewidth=0.5pt]{$\substack{\mbox{ raw data }\{ \vz^{(i)}=(\vx^{(i)},y^{(i)}) \}_{i=1}^{N}}{\mbox{stored in the cloud}}}$}
			\rput(2.5,-0.7){\psovalbox[linecolor = black,linewidth=0.5pt]{\parbox[c]{4.7cm}{raw training data $\dataset=\{\vd^{(i)},y^{(i)} \}_{i=1}^{N}$.}}}
		\end{pspicture}
	\end{center}
	\vspace*{3mm}
\caption{Programming model underlying a FOM.}
\label{fig_FOM}
\end{figure}
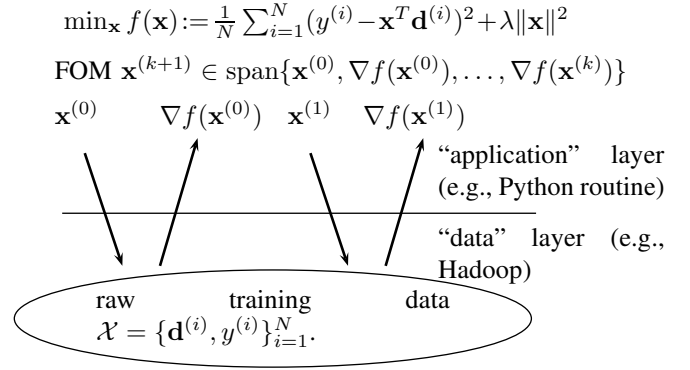

There has emerged a widely accepted computational model for convex optimization which abstracts 
away the details of the computational (hard- and software) infrastructure. Within this computational 
model, an optimization method for solving \eqref{equ_opt_problem} is not provided with a complete 
description of the objective function, but rather it can access the objective function only via an ``oracle'' \cite{nestrov04,Cevher14}. 

We might think of an oracle model as an application programming interface (API), which specifies the format 
of queries which can be issued by a convex optimization method executed on an application layer (cf.\ Figure \ref{fig_FOM}). 
There are different types of oracle models but one of the most popular type (in particular for big data applications) 
is a first order oracle \cite{nestrov04}. Given a query point $\vx \in \mathbb{R}^{\vecdim}$, a first order 
oracle returns the gradient $\nabla f(\vx)$ of the objective function at this particular point. 

A first order method (FOM) aims at solving \eqref{equ_opt_problem} by sequentially querying a first 
order oracle, at the current iterate $\vx^{(k)}$, to obtain the gradient $\nabla f(\vx^{(k)})$ (cf.\ Figure \ref{fig_FOM}). 
Using the current and past information obtained from the oracle, a FOM then constructs the new 
iterate $\vx^{(k+1)}$ such that eventually $\lim_{k \rightarrow \infty} \vx^{(k)} = \vx_{0}$.
For the sake of simplicity and without essential loss in generality, we will only 
consider FOMs whose iterates $\vx^{(k)}$ satisfy \cite{nestrov04}
\begin{equation}
\label{equ_first_order_method}
\hspace*{-5mm}\vx^{(k)} \!\in\! {\rm span} \big\{ \vx^{(0)},\nabla f(\vx^{(0)}), \ldots,\nabla f(\vx^{(k-1)}) \big\}. 
\end{equation}

\section{Lower Bounds on Number of Iterations}
\label{sec_lower_bound} 

According to Section \ref{sec_gradient_descent}, solving \eqref{equ_opt_problem} can be 
accomplished by the simple GD iterations \eqref{equ_iteration_GD}. The particular choice 
$\alpha^{*}$ \eqref{equ_opt_step_size} for the step size $\alpha$ in \eqref{equ_iteration_GD} 
ensures the convergence rate $\big(\frac{\kappa\!-\!1}{\kappa\!+\!1}\big)^{k}$ 
with the condition number $\kappa=U/L$ of the function class $\funclass{L}{U}$. 
While this convergence is quite fast, i.e., the error decays exponentially 
with iteration number $k$, we would, of course, like to know how efficient this method is in general. 

As detailed in Section \ref{sec_FOM}, in order to study the computational complexity and efficiency 
of convex optimization methods, we have to define a computational model such as those underlying 
FOMs (cf.\ Figure \ref{fig_FOM}). 
The next result provides a fundamental lower bound on the convergence rate of any FOM 
(cf.\ \eqref{equ_first_order_method}) for solving \eqref{equ_opt_problem}.
\begin{theorem}
\label{thm_lower_bound_FOM}
Consider a particular FOM, which for a given convex function $f(\cdot) \in \funclass{L}{U}$  
generates iterates $\vx^{(k)}$ satisfying \eqref{equ_first_order_method}. % for solving \eqref{equ_opt_problem}. 
For fixed $L,U$ there is a sequence of functions $f_{\vecdim}(\cdot) \in \funclass{L}{U}$ 
(indexed by dimension $\vecdim$) such that   
%Then, there is a function $f(\cdot) \in \funclass{L}{U}$ such that the 
%optimization error is lower bounded as
\begin{equation}
\label{equ_lower_bound}
\| \vx^{(k)}\!-\!\vx_{0} \| \geq  \| \vx^{(0)}\!-\!\vx_{0} \|  \frac{1\!-\!1/\sqrt{\kappa}}{1\!+\!\sqrt{\kappa}} \bigg(\frac{\sqrt{\kappa}\!-\!1}{\sqrt{\kappa}\!+\!1}\bigg)^{k} - |\delta(\vecdim)|
\end{equation} 
with a sequence $\delta(n)$ such that $\lim_{\vecdim \rightarrow \infty} |\delta(\vecdim)| =0$. 
\end{theorem}
\begin{proof}
see Section \ref{proof_them_lower_bound}.
\end{proof}
 
There is a considerable gap between the upper bound \eqref{equ_upper_bound_GD} 
on the error achieved by GD after $k$ iterations and the lower bound \eqref{equ_lower_bound} 
which applies to any FOM which is run for the same number iterations. 
In order to illustrate this gap, we have plotted in Figure \ref{fig_gap_upper_lower} 
the upper and lower bound for the (quite moderate) condition number $\kappa\!=\!100$. 
\begin{figure}
\begin{center}
\hspace*{2mm}
\psset{xunit=0.7cm,yunit=2cm}
\begin{pspicture}(-0.4,-0.1)(10.8,2.2)
\psaxes[labels=none]{->}(0,0)(-0.3,-0.1)(10.8,2.2)
%\psline[linewidth=1pt]{}(0,2)(2,0)
%\psline[linewidth=1pt]{}(2,0)(4,2)
 \psplot[algebraic,linecolor=red,linewidth=2pt]{0}{10}{2*(9/11)^(10*x)}
  \psplot[algebraic,linecolor=red,linewidth=2pt,linestyle=dashed]{0}{10}{2*(99/101)^(10*x)}
%\psline[linewidth=1pt,linestyle=dashed]{}(0,2)(4,1)
%\psline[linewidth=1pt,linestyle=dotted]{}(0,1.2)(3.2,1.2)
%\rput[tl](-0.9,1.4){$q^{*}\!=\!\frac{\kappa-1}{\kappa+1}$}
%\psline[linewidth=1pt,linestyle=dotted]{}(3.2,1.2)(3.2,0)
%\rput[tl](3.0,-0.05){$\alpha^{*}\!=\!\frac{2}{L+U}$}
%\rput[tl](-0.4,2){$1$}
%\rput[tl](2,-0.2){$1/U$}
\rput[tl](11,0){$k$}
\rput[Br](-0.5,1){$1/2$}
\rput[Br](5.4,-0.3){$50$}
\rput[Br](10.4,-0.3){$100$}
\rput[Br](-0.5,2){$1$}
\rput[tl](5,2){$\kappa\!=\!100$}
\rput[tl](5,1.1){$\big(\frac{\kappa-1}{\kappa+1} \big)^{k}$ (GD error)}
\rput[tl](2,0.32){$\big(\frac{\sqrt{\kappa}-1}{\sqrt{\kappa}+1} \big)^{k}$ (lower bound)}
%\rput[tl](4,1){$|1\!-\!L \alpha|$}
\end{pspicture}
\end{center}
\caption{Upper bound \eqref{equ_upper_bound_GD} on convergence rate of GD 
and lower bound \eqref{equ_lower_bound} on convergence rate for any FOM 
minimizing functions $f(\cdot) \!\in\! \funclass{L}{U}$ with condition number $\kappa\!=\!U/L\!=\!100$.}
\label{fig_gap_upper_lower}
\end{figure}
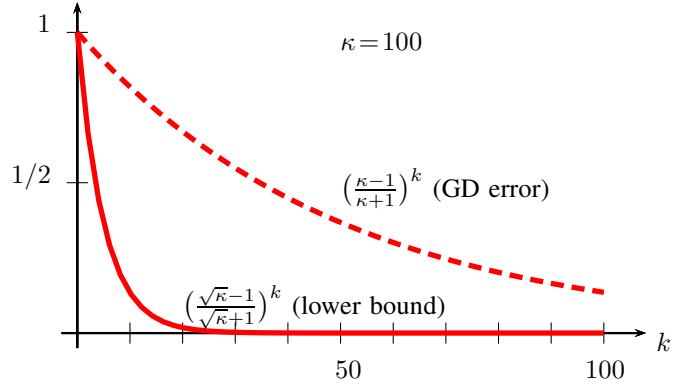

Thus, there might exist a FOM which converges faster than the GD method 
\eqref{equ_upper_bound_GD} and comes more close to the lower bound \eqref{equ_lower_bound}. 
Indeed, in the next section, we will detail how to obtain an accelerated FOM by applying a 
fixed point preserving transformation to the operator $\gdopopt$ (cf.\ \eqref{equ_fixed_point_iterations}), 
which is underlying the GD method \eqref{equ_iteration_GD}. This accelerated gradient 
method is known as the heavy balls (HB) method \cite{Polyak64} and effectively achieves 
the lower bound \eqref{equ_lower_bound}, i.e., the HB method is already optimal among 
all FOM's for solving \eqref{equ_opt_problem} with an objective function $f(\cdot) \in \funclass{L}{U}$.

\section{Accelerating Gradient Descent} 
\label{sec_AGD}

Let us now show how to modify the basic GD method \eqref{equ_iteration_GD} 
in order to obtain an accelerated FOM, whose convergence rate essentially matches the 
lower bound \eqref{equ_lower_bound} for the function class $\funclass{L}{U}$ with 
condition number $\kappa\!=\!U/L\!>\!1$ (cf.\ \eqref{equ_def_condition_number}) and is therefore optimal among all FOMs. 

Our derivation of this accelerated gradient method, which is inspired by the techniques used in \cite{GadimiShames2013}, 
starts from an equivalent formulation of GD as the fixed-point iteration 
\begin{equation}
\label{equ_stacked_fixed_point}
\bar{\vx}^{(k)} = \stackedgdop \bar{\vx}^{(k-1)} 
\end{equation} 
with the operator 
\begin{equation}
\label{equ_def_stackop}
 \hspace*{-3mm}\stackedgdop\!:\!\mathbb{R}^{2\vecdim} \!\rightarrow\! \mathbb{R}^{2\vecdim}: 
 \begin{pmatrix} \vu \\ \vv \end{pmatrix} \!\mapsto\!  \begin{pmatrix}  \vu\!-\!\alpha \nabla \vu \\  \vu  \end{pmatrix}
 =\begin{pmatrix}  \gdopopt \vu \\  \vu  \end{pmatrix}.
\end{equation}

As can be verified easily, the fixed-point iteration \eqref{equ_stacked_fixed_point} starting from an arbitrary initial guess 
$\bar{\vx}^{(0)} = \begin{pmatrix} \vz^{(0)} \\ \vy^{(0)} \end{pmatrix}$ is related to the GD iterate $\vx^{(k)}$ (cf.\ \eqref{equ_iteration_GD}), 
using initial guess $\vz^{(0)}$, as
\begin{equation}
\label{equ_equiv}
\bar{\vx}^{(k)} = \begin{pmatrix}  \vx^{(k)} \\ \vx^{(k-1)} \end{pmatrix}
\end{equation}
for all iterations $k \geq 1$. 

By the equivalence \eqref{equ_equiv}, Theorem \ref{equ_theorem_GD_convergence} implies 
that for any initial guess $\bar{\vx}^{(0)}$ the iterations \eqref{equ_stacked_fixed_point} 
converge to the fixed point 
\begin{equation} 
\label{eq_fixed_points_stacked}
\bar{\vx}_{0} \defeq \begin{pmatrix} \vx_{0} \\ \vx_{0} \end{pmatrix} \in \mathbb{R}^{2 \vecdim} 
\end{equation} 
with $\vx_{0}$ being the unique minimizer of $f(\cdot) \in \funclass{L}{U}$. 
Moreover, the convergence rate of the fixed-point iterations \eqref{equ_stacked_fixed_point} is precisely  
the same as those of the GD method, i.e., governed by the decay of $\big(\frac{\kappa-1}{\kappa+1}\big)^{k}$, 
which is obtained for the optimal step size $\alpha=\alpha^{*}$ (cf.\ \eqref{equ_opt_step_size}). 

We will now modify the operator $\stackedgdop$ in \eqref{equ_def_stackop} to obtain a new operator 
$\mathcal{M}: \mathbb{R}^{2\vecdim} \!\rightarrow\! \mathbb{R}^{2\vecdim}$ which has the 
same fixed points \eqref{eq_fixed_points_stacked} but improved contraction behaviour, i.e., the 
fixed point iteration
\begin{equation}
\label{equ_fixed_point_acc}
\tilde{\vx}^{(k)} = \mathcal{M} \tilde{\vx}^{(k-1)}, 
\end{equation}
will converge faster than those obtained from $\stackedgdop$ in \eqref{equ_stacked_fixed_point}. 
In particular, this improved operator $\mathcal{M}$ is defined as 
\begin{equation}
\label{equ_def_AGM}
 \hspace*{-3mm}\mathcal{M}\!:\!\mathbb{R}^{2\vecdim} \!\rightarrow\! \mathbb{R}^{2\vecdim}: \begin{pmatrix} \vu \\ \vv \end{pmatrix} \!\mapsto\!  
 \begin{pmatrix}  \vu\!-\!\tilde{\alpha} \nabla \vu + \tilde{\beta} (\vu - \vv) \\  \vu  \end{pmatrix}, 
\end{equation}
with 
\begin{equation}
\label{equ_def_tilde_apha_beta}
\tilde{\alpha} \defeq \frac{4}{(\sqrt{U}\!+\!\sqrt{L})^2}  \mbox{, and } \tilde{\beta} \defeq  \bigg[ \frac{\sqrt{U}\!-\!\sqrt{L}}{\sqrt{U}\!+\!\sqrt{L}} \bigg]^{2}.   
\end{equation}
As can be verified easily, the fixed point $\big(\vx^{T}_{0},\vx^{T}_{0}\big)^{T}$ of $\stackedgdop$ is also a fixed point of 
$\mathcal{M}$.
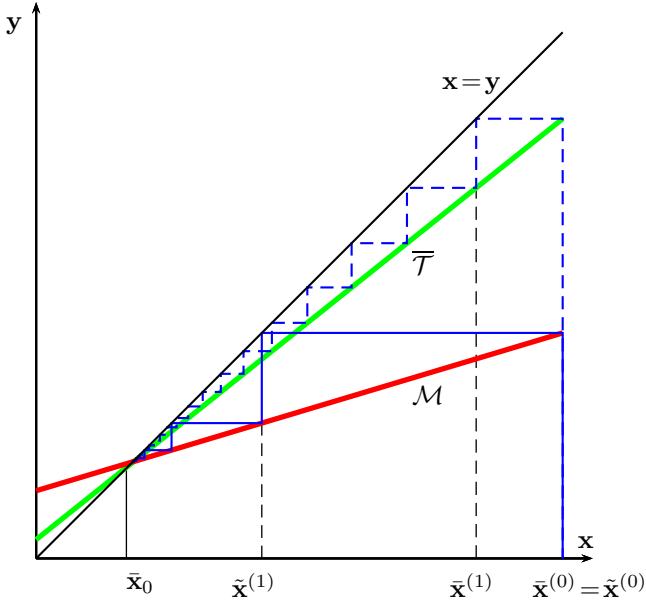
\begin{figure}
\begin{pspicture}[algebraic](-5mm,-1cm)(7,7)
%\psset{xunit=0.7cm,yunit=1cm}
  \psaxes[labels=none,ticks=none]{->}(7.4,7.4)
  \rput[tl](5.4,6.4){$\vx\!=\!\vy$}
\rput[tl](6.6,-0.2){$\bar{\vx}^{(0)}\!=\!\tilde{\vx}^{(0)}$}
\rput[tl](2.6,-0.2){$\tilde{\vx}^{(1)}$}
 \rput[tl](7.2,0.3){$\vx$}
 \rput[tl](-0.4,7.2){$\vy$}
 \rput[tl](5.5,-0.2){$\bar{\vx}^{(1)}$}
%  \rput[tl](-0.6,5.3){$\bar{\vx}^{(2)}$}
   \rput[tl](1.2,-0.2){$\bar{\vx}_{0}$}
      \psline[linewidth=0.5pt,linestyle=dashed](3,0)(3,1.8)
%   \psline[linewidth=0.5pt,linestyle=dashed](0,4.95)(5,4.95)
\psline[linewidth=0.5pt](1.2,1.2)(1.2,0)
\psline[linewidth=0.5pt,linestyle=dashed](5.85,0)(5.85,4.95)
  \psplot[linecolor=green,linewidth=2pt]{0}{7}{5/20+8*x/10}
  \rput[tl](5,4.1){$\stackedgdop$}
    \rput[tl](5,2.3){$\mathcal{M}$}
    \psplot[linecolor=red,linewidth=2pt]{0}{7}{9/10+3*x/10}
      \psFixpoint[linecolor=blue]{7}{9/10+3*x/10}{20}
            \psFixpoint[linecolor=blue,linestyle=dashed]{7}{5/20+8*x/10}{20}
  \psline(7,7)
\label{equ_fig_AGD_fixed_point}
\end{pspicture}
\caption{Schematic illustration of the fixed-point iteration using operator $\stackedgdop$ \eqref{equ_def_stackop} (equivalent to GD) and for the 
modified operator $\mathcal{M}$ \eqref{equ_def_AGM} (yielding HB method).}
\end{figure}

Before we analyze the convergence rate of the fixed-point iteration \eqref{equ_fixed_point_acc}, 
let us work out explicitly the FOM which is represented by the fixed-point iteration \eqref{equ_fixed_point_acc}. 
To this end, we partition the $k$th iterate, for $k \geq 1$, as 
\begin{equation} 
\label{equ_partition_HB}
\tilde{\vx}^{(k)} \defeq \begin{pmatrix} \vx_{\rm HB}^{(k)} \\ \vx_{\rm HB}^{(k-1)} \end{pmatrix}. 
\end{equation}
Inserting \eqref{equ_partition_HB} into \eqref{equ_fixed_point_acc}, we have for $k \geq 1$
\begin{equation}
\label{equ_iteration_HB}
\vx_{\rm HB}^{(k)} = \vx_{\rm HB}^{(k-1)}\!-\!\tilde{\alpha} \nabla f(\vx_{\rm HB}^{(k-1)})\!+\!\tilde{\beta} (\vx_{\rm HB}^{(k-1)}\!-\!\vx_{\rm HB}^{(k-2)})
\end{equation}
with the convention $\vx_{\rm HB}^{(-1)} \defeq \mathbf{0}$. The iteration \eqref{equ_iteration_HB} 
defines the HB method \cite{Polyak64} for solving the optimization problem \eqref{equ_opt_problem}. 
As can be verified easily, like the GD method, the HB method is a FOM. 
However, contrary to the GD iteration \eqref{equ_iteration_GD}, the HB iteration \eqref{equ_iteration_HB} 
also involves the penultimate iterate $\vx_{\rm HB}^{(k-2)}$ for determining the new iterate $\vx_{\rm HB}^{(k)}$. 

We will now characterize the converge rate of the HB method \eqref{equ_iteration_HB} 
via its fixed-point equivalent \eqref{equ_fixed_point_acc}. To this end, we restrict 
ourselves to the subclass of $\funclass{L}{U}$ given by quadratic functions 
of the form \eqref{equ_quadratic_function}. 
\begin{theorem}
\label{theorem_HB_convergence}
Consider the optimization problem \eqref{equ_opt_problem} with objective function 
$f(\cdot) \in \funclass{L}{U}$ which is a quadratic \eqref{equ_quadratic_function}. %  where the parameters $L$ and $U$ are fixed and known. 
Starting from an arbitrarily chosen initial guess $\vx_{\rm HB}^{(-1)}$ and $\vx_{\rm HB}^{(0)}$, 
we construct a sequence $\vx_{\rm HB}^{(k)}$ via iterating \eqref{equ_iteration_GD}. Then, 
\begin{equation}
\label{equ_upper_bound_HB}
 \| \vx_{\rm HB}^{(k)}\!-\! \vx_{0} \| \!\leq\!  C(\kappa) k \bigg(\frac{\sqrt{\kappa}\!-\!1}{\sqrt{\kappa}\!+\!1} \bigg)^{k} 
 (\|\vx_{\rm HB}^{(0)}\!-\! \vx_{0} \|\!+\!\|\vx_{\rm HB}^{(-1)}\!-\!\vx_{0} \|). 
\end{equation}
with 
\begin{equation}
C(\kappa) \defeq 4  (2\!+\!2 \tilde{\beta}\!+\!\tilde{\alpha}) \frac{\sqrt{\kappa}\!+\!1}{\sqrt{\kappa}\!-\!1}.  \nonumber
\end{equation}
\end{theorem}
\begin{proof}
see Section \ref{sec_proof_theorem_HB_convergence}.
\end{proof}
The upper bound \eqref{equ_upper_bound_HB} differs from the lower bound 
\eqref{equ_lower_bound} by the factor $k$. However, the discrepancy 
is rather decent as this linear factor in \eqref{equ_upper_bound_HB} grows much slower 
than the exponential $\big(\frac{\sqrt{\kappa}\!-\!1}{\sqrt{\kappa}\!+\!1} \big)^{k}$ in \eqref{equ_upper_bound_HB} decays. 
In Figure \ref{fig_gap_upper_lower_HB}, we depict the upper bound \eqref{equ_upper_bound_HB} 
on the error of the HB iterations \eqref{equ_iteration_HB} along with the upper bound \eqref{equ_upper_bound_GD} 
on the error of the GD iterations \eqref{equ_iteration_GD} and the lower bound \eqref{equ_lower_bound} 
on the error of any FOM after $k$ iterations. 

We highlight that, strictly speaking, the bound \eqref{equ_upper_bound_HB} only applies to a subclass 
of smooth strongly convex functions $f(\cdot) \in \funclass{L}{U}$, i.e., it applies 
only to quadratic functions of the form \eqref{equ_quadratic_function}. However, 
as discussed in Section \ref{sec_cvx_functions}, given a particular point $\vx_{}$, we can approximate 
an arbitrary function $f(\cdot) \in \funclass{L}{U}$ with a quadratic function $\tilde{f}(\vx)$ 
of the form \eqref{equ_quadratic_function}. The approximation error $\varepsilon(\vx)$ (cf.\ \eqref{equ_approx_error_quadratic}) 
will be small for all points $\vx$ sufficiently close to $\vx_{0}$. Making this reasoning 
more precise and using well-known results on fixed-point iterations with inexact updates \cite{Alfeld82}, 
one can verify that the bound \eqref{equ_upper_bound_HB} essentially applies to any function 
$f(\cdot) \in  \funclass{L}{U}$. 
\begin{figure}
\begin{center}
\hspace*{0mm}
\psset{xunit=0.7cm,yunit=2cm}
\begin{pspicture}(-0.4,-0.1)(10.8,2.2)
\psaxes[labels=none]{->}(0,0)(-0.3,-0.1)(10.8,2.2)
 \psplot[algebraic,linecolor=red,linewidth=2pt]{1}{10}{(1/5)*log(x*20*(9/11)^(10*x))+1.8}
  \psplot[algebraic,linecolor=red,linewidth=2pt,linestyle=dotted]{1}{10}{(1/5)*log(2*(9/11)^(10*x))+1.8}
  \psplot[algebraic,linecolor=red,linewidth=2pt,linestyle=dashed]{1}{10}{(1/5)*log(2*(99/101)^(10*x))+1.8}
\rput[tl](11,0){$k$}
\rput{90}(-0.5,1){$\log \| \vx^{(k)}\!-\!\vx_{0} \|$}
\rput[Br](5.4,-0.3){$50$}
\rput[Br](10.4,-0.3){$100$}
\rput[tl](5,2.5){$\kappa\!=\!100$}
\rput[tl](8,2){GD \eqref{equ_upper_bound_GD}}
\rput[tl](2,0.6){lower bound \eqref{equ_lower_bound}}
\rput[tl](8,1.1){HB \eqref{equ_upper_bound_HB}}
\end{pspicture}
\end{center}
\caption{Dependence on iteration number $k$ of the upper bound \eqref{equ_upper_bound_HB} on error of 
HB (solid), upper bound \eqref{equ_upper_bound_GD} for error of GD (dashed) and lower bound \eqref{equ_lower_bound} 
(dotted) for FOMs for the function class $\funclass{L}{U}$ with condition number $\kappa\!=\!U/L\!=\!100$.}
\label{fig_gap_upper_lower_HB}
\end{figure}
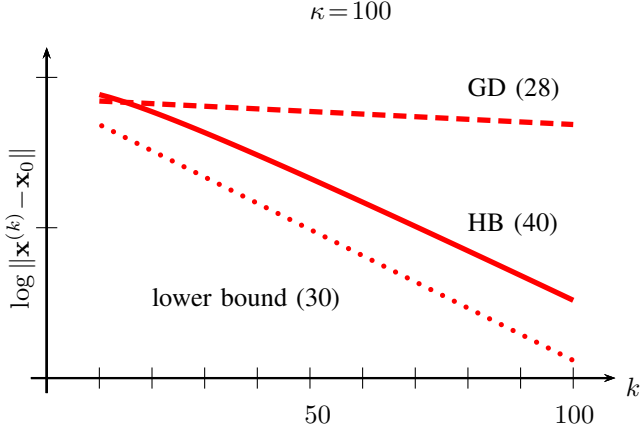

%\section{Stochastic Gradient Descent} 

%\section{Proximal Point Algorithm} 

%%%%%%%%%%%%%%%%%%%%%%%%%%
\section{Conclusions}
\label{sec5_conclusion}
We have presented a fixed-point theory of some basic gradient methods for minimizing convex 
functions. The approach via fixed-point theory allows for a rather elegant analysis of the 
convergence properties of these gradient methods. In particular, their convergence rate is 
obtained as the contraction factor for an operator associated with the objective function. 

The fixed-point approach is also appealing since it leads rather naturally to the acceleration 
of gradient methods via fixed-point preserving transformations of the underlying operator. 
We plan to further develop the fixed-point theory of gradient methods in order to accommodate 
stochastic variants of GD such as SGD. Furthermore, we can bring the popular class 
of proximal methods into the picture by replacing the gradient operator underlying GD with 
the proximal operator. 

However, by contrast to FOMs (such as the GD method), proximal methods 
use a different oracle model (cf.\ Figure \ref{fig_FOM}). In particular, proximal methods require an 
oracle which can evaluate the proximal mapping efficiently which is typically more expensive 
than gradient evaluations. Nonetheless, the popularity of proximal methods is due to the fact 
that for objective functions arising in many important machine learning applications, the proximal 
mapping can be evaluated efficiently. 

\section*{Acknowledgement}
This paper is a wrap-up of the lecture material created for the course \emph{Convex Optimization 
for Big Data over Networks}, taught at Aalto University in spring $2017$. The student feedback on 
the lectures has been a great help to develop the presentation of the contents. In particular, the 
detailed feedback of students Stefan Mojsilovic and Matthias Grezet on early versions of the paper 
is appreciated sincerely.

\section{Proofs of Main Results} 
In this section we present the (somewhat lengthy) proofs for the main results 
stated in Section \ref{sec_lower_bound} and Section \ref{sec_AGD}. 
\subsection{Proof of Theorem \ref{thm_lower_bound_FOM}} 
\label{proof_them_lower_bound}
Without loss of generality we consider FOM which use the 
initial guess $\vx^{(0)} = \mathbf{0}$. Let us now construct a function $f_{\vecdim}(\cdot)\in \funclass{L}{U}$ 
which is particularly difficult to optimize by a FOM (cf.\ \eqref{equ_first_order_method}) 
such as the GD method \eqref{equ_iteration_GD}. 
In particular, this function is the quadratic 
\begin{equation}
\label{equ_worst_func_ever}
\hat{f}(\vx) \defeq (1/2) \vx^{T}\mathbf{P} \vx + \tilde{\mathbf{q}}^{T} \vx 
\end{equation}
with vector 
\begin{equation}
\label{equ_def_tilde_vq}
\tilde{\vq} \defeq  \frac{L(\kappa\!-\!1)}{4}(1,0,\ldots,0)^{T} \in \mathbb{R}^{\vecdim}
\end{equation} 
and matrix 
\begin{equation}
\label{equ_def_matrix_P}
\mathbf{P} \defeq(L/4)(\kappa\!-\!1)\widetilde{\mathbf{Q}}\!+\!L \mathbf{I}  \in \mathbb{R}^{\vecdim \times \vecdim}. 
\end{equation}
The matrix $\widetilde{\mathbf{Q}}$ is defined row-wise by successive circular shifts of 
its first row 
\begin{equation} 
\label{equ_def_first_row_tilde_q}
\tilde{\mathbf{q}} \defeq(2,-1,0,\ldots,0,-1)^{T} \in \mathbb{R}^{\vecdim}.
\end{equation} 
Note that the matrix $\mathbf{P}$ in \eqref{equ_def_matrix_P} is a circulant matrix 
\cite{GrayToepliz} with orthonormal eigenvectors 
$\big\{ \mathbf{u}^{(l)} \big\}_{l=1}^{\vecdim}$ given element-wise as 
\begin{equation} 
\label{equ_def_DFT_vector}
u^{(l)}_{i} \!=\! (1/\sqrt{\vecdim}) \exp( j 2 \pi (i\!-\!1)(l\!-\!1)/\vecdim).
\end{equation}  
The eigenvalues $\lambda_{l}(\mathbf{P})$ of the circulant matrix $\mathbf{P}$ are obtained as 
the discrete Fourier transform (DFT) coefficients of its first row \cite{GrayToepliz}
\begin{align} 
\label{equ_first_row_p}
\mathbf{p}& \!=\!\frac{L(\kappa\!-\!1)}{4}\tilde{\vq} \!+\!L\mathbf{e}_{1}^{T} \nonumber \\ 
& \!\stackrel{\eqref{equ_def_first_row_tilde_q}}{=}\!\frac{L(\kappa\!-\!1)}{4}(2,-1,0,\ldots,0,-1)\!+\!L(1,0,\ldots,0)^{T},
\end{align}
i.e., 
\begin{align}
\label{equ_eigvals_bad_matrix}
\lambda_{l}(\mathbf{P}) & = \sum_{i=1}^{\vecdim} p_{i}  \exp(- j 2 \pi (i\!-\!1)(l\!-\!1) / \vecdim) \\ \nonumber 
&  \stackrel{\eqref{equ_first_row_p}}{=} (L/2)(\kappa\!-\!1)(1\!-\!\cos (- 2 \pi (i\!-\!1) / \vecdim) + L.    
\end{align} 
Thus, $\lambda_{l}(\mathbf{P}) \in [L,U]$ and, in turn, $f_{\vecdim}(\cdot) \in \funclass{L}{U}$ (cf.\ \eqref{double_bound_hessian}). 

Consider the sequence $\vx^{(k)}$ generated by some FOM, i.e., which satisfies \eqref{equ_first_order_method}, 
for the particular objective function $f_{\vecdim}(\vx)$ (cf.\ \eqref{equ_worst_func_ever}) using initial guess $\vx_{0} = \mathbf{0}$. 
It can be verified easily that the $k$th iterate $\vx^{(k)}$ has only zero entries starting from index $k+1$, i.e., 
\begin{equation}
x^{(k)}_{l} = 0 \mbox{ for all } l \in \{k+1,\ldots,\vecdim\}.  \nonumber
\end{equation} 
This implies 
\begin{equation}
\label{equ_lower_bound_xo_entry}
\| \vx^{(k)} - \vx_{0} \| \geq |x_{0,k+1}|.
\end{equation} 
The main part of the proof is then to show that the minimizer $\vx_{0}$ for the 
particular function $f_{\vecdim}(\cdot)$ cannot decay too fast, i.e., we will derive 
a lower bound on $ |x_{0,k+1}|$. 

Let us denote the DFT coefficients of the finite length discrete time signal represented by the vector 
$\tilde{\mathbf{q}}$ as 
\begin{align}
\label{equ_DFT_bad_vector}
c_{l} & = \sum_{i=1}^{\vecdim} \tilde{q}_{i} \exp(- j 2 \pi (i-1)l / \vecdim) \nonumber \\ 
        & \stackrel{\eqref{equ_def_tilde_vq}}{=} (L/4)(\kappa-1).
\end{align} 
Using the optimality condition \eqref{equ_zero_gradient}, 
the minimizer for \eqref{equ_worst_func_ever} is  
\begin{equation} 
\label{equ_close_form_solution}
\vx_{0} = -  \mathbf{P}^{-1} \tilde{\mathbf{q}}. 
\end{equation} 
Inserting the spectral decomposition $\mathbf{P}= \sum\limits_{l=1}^{\vecdim} \lambda_{l} \vu^{(l)} \big(\vu^{(l)}\big)^{H}$ \cite[Theorem 3.1]{GrayToepliz}
of the psd matrix $\mathbf{P}$ into \eqref{equ_close_form_solution}, 
\begin{align}
x_{0,k} & = - \big( \mathbf{P}^{-1} \tilde{\mathbf{q}} \big)_{k} \nonumber \\ 
		      & \stackrel{\eqref{equ_def_DFT_vector}}{=} -(1/\vecdim) \sum_{i=1}^{\vecdim} (c_{i}/\lambda_{i})  \exp(j 2\pi (i\!-\!1) (k\!-\!1)/\vecdim) \nonumber \\ 
		      & \hspace*{-12mm}\stackrel{\eqref{equ_eigvals_bad_matrix},\eqref{equ_DFT_bad_vector}}{=} -\frac{1}{\vecdim} 
		      \sum_{i=1}^{\vecdim}  \frac{ \exp(j 2\pi (i\!-\!1) (k\!-\!1)/\vecdim)}{ 2(1\!-\!\cos (- 2 \pi (i\!-\!1) / \vecdim))\!+\!4/(\kappa\!-\!1)}. \label{equ_proof_lower_bound_112}
\end{align}
We will also need a lower bound on the norm $\| \vx_{0} \|$ of the minimizer of $f_{\vecdim}(\cdot)$. 
This bound can be obtained from \eqref{equ_close_form_solution} and $\lambda_{l}(\mathbf{P})\!\in\![L,U]$, i.e., 
$\lambda_{l} \big(\mathbf{P}^{-1} \big)\!\in\![1/U,1/L]$,  
\begin{equation}
\label{equ_lower_bound_x_0}
\| \vx_{0} \| \leq (1/L) \| \tilde{\vq}\| \stackrel{\eqref{equ_def_tilde_vq}}{=}  \frac{\kappa\!-\!1}{4}.
\end{equation}

The last expression in \eqref{equ_proof_lower_bound_112} is 
a Riemann sum for the integral $\int\limits_{\theta=0}^{1}  \frac{\exp(-j2\pi\theta)}{ 2(1\!-\!\exp(-j2\pi \theta))\!+\!4/(\kappa-1)}  d \theta$. 
Indeed, by basic calculus \cite[Theorem 6.8]{RudinBookPrinciplesMatheAnalysis}
\begin{align}
\label{equ_approx_solution_integeral}
x_{0,k} &= -\int\limits_{\theta=0}^{1}  \hspace*{-2mm}\frac{\exp(j2\pi (k\!-\!1) \theta)}{ 2(1\!-\!\cos(2\pi \theta))\!+\!4/(\kappa\!-\!1)}  d \theta \!+\! \delta(\vecdim)
\end{align} 
where the error $\delta(\vecdim)$ becomes arbitrarily small for 
sufficiently large $\vecdim$, i.e., $\lim\limits_{\vecdim \rightarrow \infty} |\delta(\vecdim)| = 0$.

According to Lemma \ref{lem_identiy_integral}, 
\begin{align} 
%\label{equ_closed_form_integral}
\hspace*{-2mm}\int\limits_{\theta=0}^{1}  \hspace*{-2mm}\frac{\exp(j2\pi (k\!-\!1) \theta)}{ 2(1\!-\!\cos(2\pi \theta))\!+\!4/(\kappa\!-\!1)}  d \theta\!=\!\frac{\kappa\!-\!1}{4 \sqrt{\kappa}} \bigg(\hspace*{-1mm}\frac{\sqrt{\kappa}\!-\!1}{\sqrt{\kappa}\!+\!1}\bigg)^{k},  \nonumber
%\int\limits_{\theta=0}^{1}  \hspace*{-2mm}\frac{\exp(-j2\pi k \theta) d \theta}{ 2(1\!-\!\exp(-j2\pi \theta))\!+\!4/(\kappa-1)}\!=\!\bigg(\frac{\sqrt{\kappa}-1}{\sqrt{\kappa}+1}\bigg)^{k}. 
\end{align}  
which, by inserting into \eqref{equ_approx_solution_integeral}, yields
\begin{equation}
\label{equ_bound_145656}
x_{0,k} = -\frac{\kappa\!-\!1}{4 \sqrt{\kappa}} \bigg(\frac{\sqrt{\kappa}\!-\!1}{\sqrt{\kappa}\!+\!1}\bigg)^{k}\!+\! \delta(\vecdim). 
\end{equation}
Putting together the pieces, 
\begin{align}
\| \vx^{(k)}\!-\!\vx_{0} \| & \stackrel{\eqref{equ_lower_bound_xo_entry}}{\geq} | x_{0,k+1} | \nonumber  \\[2mm]  
    & \hspace*{-15mm} \stackrel{\eqref{equ_bound_145656}}{\geq}  \frac{\kappa\!-\!1}{4 \sqrt{\kappa}} \bigg(\frac{\sqrt{\kappa}\!-\!1}{\sqrt{\kappa}\!+\!1}\bigg) \bigg(\frac{\sqrt{\kappa}\!-\!1}{\sqrt{\kappa}\!+\!1}\bigg)^{k}\!-\!|\delta(\vecdim)| \nonumber \\[2mm] 
    & \hspace*{-15mm} \stackrel{\eqref{equ_lower_bound_x_0}}{\geq}  \| \vx_{0} \| \frac{1\!-\!1/\sqrt{\kappa}}{1\!+\!\sqrt{\kappa}} \bigg(\frac{\sqrt{\kappa}\!-\!1}{\sqrt{\kappa}\!+\!1}\bigg)^{k}\!-\!|\delta(\vecdim)| \nonumber \\[2mm] 
    & \hspace*{-15mm} \stackrel{\vx^{(0)} = \mathbf{0}}{=}  \| \vx^{(0)}\!-\!\vx_{0} \| \frac{1\!-\!1/\sqrt{\kappa}}{1\!+\!\sqrt{\kappa}} \bigg(\frac{\sqrt{\kappa}\!-\!1}{\sqrt{\kappa}\!+\!1}\bigg)^{k}\!-\!|\delta(\vecdim)|.\nonumber
\end{align} 

\subsection{Proof of Theorem \ref{theorem_HB_convergence}}
\label{sec_proof_theorem_HB_convergence}

By evaluating the operator $\mathcal{M}$ (cf.\ \eqref{equ_def_AGM}) for a quadratic function $f(\cdot)$ 
of the form \eqref{equ_quadratic_function}, we can verify  
\begin{equation}
\label{equ_contraction_inequ_M}
\mathcal{M} \vx\!-\!\mathcal{M} \vy  = \mathbf{R} (\vx\!-\!\vy) 
\end{equation} 
with the matrix 
\begin{equation}
\label{equ_def_AGM_diff}
\mathbf{R}  =  \begin{pmatrix} (1\!+\!\tilde{\beta})\mathbf{I}\!-\!\tilde{\alpha} \mathbf{Q} & - \tilde{\beta} \mathbf{I} \\ \mathbf{I} & \mathbf{0} \end{pmatrix}. 
\end{equation}
%We now derive an upper bound on the spectral norm $\| \mathbf{R}  \|$ by exploiting 
This matrix $\mathbf{R}\in \mathbb{R}^{2 \vecdim \times 2 \vecdim}$ is a $2 \times 2$ block matrix whose individual blocks can be 
diagonalized simultaneously via the orthonormal eigenvectors $\mathbf{U}=\big(\vu^{(1)},\ldots,\vu^{(\vecdim)} \big)$ of the 
psd matrix $\mathbf{Q}$. Inserting the spectral decomposition $\mathbf{Q}\!=\!\mathbf{U} {\rm diag} \{ \lambda_{i} \}_{i=1}^{\vecdim} \mathbf{U}^{H}$ 
into \eqref{equ_def_AGM_diff},  
\begin{equation}
\label{equ_factorization_diff}
\mathbf{R}= \mathbf{U} \mathbf{P} \mathbf{B} \mathbf{P}^{H} \mathbf{U}^{H}, 
\end{equation}
with some (orthonormal) permutation matrix $\mathbf{P}$ and a block diagonal matrix 
\begin{equation}
\label{equ_def_block_diagonalB}
\mathbf{B}\!\defeq\!\begin{pmatrix} \mathbf{B}^{(1)} &  \ldots  & \mathbf{0} \\ 
\mathbf{0} & \ddots & \vdots \\
\mathbf{0} & \ldots & \mathbf{B}^{(\vecdim)} \end{pmatrix} 
\mbox{, } \mathbf{B}^{(i)}\!\defeq\!\begin{pmatrix} 1\!+\!\tilde{\beta}\!-\!\tilde{\alpha} \lambda_{i} & - \tilde{\beta} \\ 1 & 0 \end{pmatrix}. 
\end{equation} 
Combining \eqref{equ_factorization_diff} with \eqref{equ_contraction_inequ_M} 
and inserting into \eqref{equ_fixed_point_acc} yields 
\begin{equation} 
\label{equ_identiy_tilde_x_k_0}
\tilde{\vx}^{(k)}\!-\!\tilde{\vx}_{0}\!=\!\mathbf{U} \mathbf{P} \mathbf{B}^{k} \mathbf{P}^{H} \mathbf{U}^{H} (\tilde{\vx}^{(0)}\!-\!\tilde{\vx}_{0} ). 
\end{equation} 
In order to control the convergence rate of the iterations \eqref{equ_fixed_point_acc}, i.e., 
the decay of the error $\| \tilde{\vx}^{(k)}\!-\!\tilde{\vx}_{0}  \|$, we will now derive an upper bound 
on the spectral norm of the block diagonal matrix $\mathbf{B}^{k}$ (cf. \eqref{equ_def_block_diagonalB}). 

Due to the block diagonal structure \eqref{equ_def_block_diagonalB}, we can control the norm of 
$\mathbf{B}^{k}$ via controlling the norm of the powers of its diagonal blocks $\big(\mathbf{B}^{(i)}\big)^{k}$ since 
\begin{equation} 
\label{equ_power_B_k_B_i}
\| \mathbf{B}^{k} \| = \max_{i} \big\| \big( \mathbf{B}^{(i)} \big)^{k} \big\|. 
\end{equation} 
A pen and paper exercise reveals % that the eigenvalues $\lambda_{l}\big(\mathbf{B}^{(i)} \big)$ satisfy 
\begin{equation}
\label{equ_upper_bound_lambda_b_Bi}
 \rho \big(\mathbf{B}^{(i)}  \big)=  \tilde{\beta}^{1/2} \stackrel{\eqref{equ_def_tilde_apha_beta}}{=}  \frac{\sqrt{U}\!-\!\sqrt{L}}{\sqrt{U}\!+\!\sqrt{L}} =\frac{\sqrt{\kappa}\!-\!1}{\sqrt{\kappa}\!+\!1}. 
\end{equation}
Combining \eqref{equ_upper_bound_lambda_b_Bi} with Lemma \ref{lem_power_M_decomp} yields
\begin{equation}
\label{equ_expr_B_i_bounding}
\big( \mathbf{B}^{(i)} \big)^{k} = \begin{pmatrix} \lambda_{1}^{k} & d \\ 0 & \lambda_{2}^{k} \end{pmatrix}, 
\end{equation} 
with $|\lambda_{1}|,|\lambda_{2}| \leq  \tilde{\beta}^{1/2}$ and $d \leq k (2\!+\!2\tilde{\beta}\!+\!\tilde{\alpha}) \tilde{\beta}^{(k-1)/2}$. 
Using the shorthand $\tilde{c} \defeq (2\!+\!2\tilde{\beta}\!+\!\tilde{\alpha})$, we can estimate the 
spectral norm of $\mathbf{B}^{k}$ as 
\begin{align}
\label{equ_bond_B_k_111}
\| \mathbf{B}^{k} \| & \stackrel{\eqref{equ_power_B_k_B_i}}{=} \max_{i} \big\| \big( \mathbf{B}^{(i)} \big)^{k} \big\| \nonumber \\ 
& \stackrel{\eqref{equ_expr_B_i_bounding}}{\leq}  \bigg(\frac{\sqrt{\kappa}\!-\!1}{\sqrt{\kappa}\!+\!1} \bigg)^{k} \bigg(1\!+\!k \tilde{c} \frac{\sqrt{\kappa}\!+\!1}{\sqrt{\kappa}\!-\!1}\bigg). 
\end{align} 
Combining \eqref{equ_bond_B_k_111} with \eqref{equ_identiy_tilde_x_k_0}, 
\begin{align}
\label{error_bound_fixed_point_HB}
\| \tilde{\vx}^{(k)}\!-\!\tilde{\vx}_{0} \| & \!\leq\!  \bigg(\frac{\sqrt{\kappa}\!-\!1}{\sqrt{\kappa}\!+\!1} \bigg)^{k} \bigg(1\!+\!k \tilde{c} \frac{\sqrt{\kappa}\!+\!1}{\sqrt{\kappa}\!-\!1}\bigg) \|\tilde{\vx}^{(0)}\!-\!\tilde{\vx}_{0} \| \nonumber \\ 
& \!\stackrel{\tilde{c} \geq1}{\leq}\!  2\!k \tilde{c} \frac{\sqrt{\kappa}\!+\!1}{\sqrt{\kappa}\!-\!1} \bigg(\frac{\sqrt{\kappa}\!-\!1}{\sqrt{\kappa}\!+\!1} \bigg)^{k}  \|\tilde{\vx}^{(0)}\!-\!\tilde{\vx}_{0} \|. 
\end{align}
Using \eqref{equ_partition_HB}, the error bound \eqref{error_bound_fixed_point_HB} can be 
translated into an error bound on the HB iterates $\vx_{\rm HB}^{(k)}$ , i.e., 
\begin{align}
 & \| \vx_{\rm HB}^{(k)}\!-\! \vx_{0} \|  \!\leq\!  \nonumber \\
 & 4\!k \tilde{c} \frac{\sqrt{\kappa}\!+\!1}{\sqrt{\kappa}\!-\!1} \bigg(\frac{\sqrt{\kappa}\!-\!1}{\sqrt{\kappa}\!+\!1} \bigg)^{k} 
 (\|\vx_{\rm HB}^{(0)}\!-\! \vx_{0} \|\!+\!\|\vx_{\rm HB}^{(-1)}\!-\!\vx_{0} \|). \nonumber
\end{align}

\section{Technicalities}
\label{sec_tech}

We collect some elementary results from linear algebra and analysis, which are required 
to prove our main results. 
\begin{lemma}
\label{lem_power_M_decomp}
Consider a matrix $\mathbf{M}\!=\!\begin{pmatrix} a & b \\ 1 & 0 \end{pmatrix}\!\in\!\mathbb{R}^{2 \times 2}$ with 
spectral radius $\rho(\mathbf{M})$. 
Then, there is an orthonormal matrix $\mathbf{U}\!\in\!\mathbb{C}^{2 \times 2}$ such that
\begin{equation}
\label{power_M_decomp}
\mathbf{M}^{k} = \mathbf{U}  \begin{pmatrix} \lambda_{1}^{k}  & d \\ 0 & \lambda_{2}^{k} \end{pmatrix} \mathbf{U}^{H} \mbox{  for } k\!\in\!\mathbb{N},
\end{equation} 
where $|\lambda_{1}|,|\lambda_{2}|\!\leq\!\rho(\mathbf{M})$ and $|d|\!\leq\!k (|a|\!+\!|b|\!+\!1)\rho^{k-1}(\mathbf{M})$.
\end{lemma} 
\begin{proof}
Consider an eigenvalue $\lambda_{1}$ of the matrix $\mathbf{M}$ with normalized eigenvector 
$\mathbf{u}\!=\!(u_{1},u_{2})^{H}$, i.e., $\mathbf{M} \vu\!=\!\lambda_{1} \vu$ with $\| \mathbf{u} \|\!=\!1$. 
According to \cite[Lemma 7.1.2]{golub96}, we can find a normalized vector $\vv\!=\!(v_{1},v_{2})^{H}$, 
orthogonal to $\vu$, such that 
\begin{equation}
\label{equ_decomp_M_111}
\mathbf{M}= (\vu,\vv)  \begin{pmatrix} \lambda_{1} & d \\ 0 & \lambda_{2} \end{pmatrix} (\vu,\vv)^{H}, 
\end{equation} 
or equivalently
\begin{equation} 
\label{equ_relation_decomp_M_123}
 \begin{pmatrix} \lambda_{1} & d \\ 0 & \lambda_{2} \end{pmatrix} = (\vu,\vv)^{H} \mathbf{M}  (\vu,\vv),
\end{equation} 
with some eigenvalue $\lambda_{2}$ of $\mathbf{M}$. 
As can be read off \eqref{equ_relation_decomp_M_123}, $d = u_{1} (u_{2} a + v_{2}b) + v_{1}u_{2}$ which implies \eqref{power_M_decomp} since 
$|u_{1}|,|u_{2}|,|v_{1}|,|v_{2}| \leq 1$.  
Based on \eqref{equ_decomp_M_111}, we can verify \eqref{power_M_decomp} by induction. 
%Let us decompose $\mathbf{M} = \mathbf{M}_{1} + \mathbf{M}_{2}$ with 
%$\mathbf{M}_{1} = \begin{pmatrix} \lambda_{1} & d \\ 0 & \lambda_{2} \end{pmatrix}$ and 
%$\mathbf{M}_{2} =  \begin{pmatrix} 0 & d \\ 0 & 0 \end{pmatrix}$. 
%Since $\mathbf{M}_{2}^{l} = \mathbf{0}$ for any $l \in \{2,3,\ldots\}$, 
%\begin{equation}
%\mathbf{M}^{k} = (\mathbf{M}_{1}\!+\!\mathbf{M}_{2})^{k} = \sum_{l=0}^{k} \binom{k}{l} \mathbf{M}^{k-l}_{1}  \mathbf{M}^{l}_{2} 
%\end{equation}
\end{proof}

\begin{lemma}
\label{lem_identiy_integral}
For any $\kappa > 1$ and $k \in \mathbb{N}$, 
\begin{equation}
\label{equ_lemma_identiy}
\hspace*{-3mm}\int\limits_{\theta=0}^{1}  \hspace*{-2mm}\frac{\exp(j2\pi (k\!-\!1) \theta)}{ 2(1\!-\!\cos(2\pi \theta))\!+\!4/(\kappa\!-\!1)}  d \theta\!=\! \frac{\kappa\!-\!1}{4 \sqrt{\kappa}} \bigg(\frac{\sqrt{\kappa}\!-\!1}{\sqrt{\kappa}\!+\!1}\bigg)^{k}. 
\end{equation} 
%$\begin{pmatrix} a & b \\ 1 & 0 \end{pmatrix}  \in \mathbb{R}^{2 \times 2}$ 
%with eigenvalues $\lambda_{1}$ and $\lambda_{2}$ (which might coincide!). 
%Then, we obtain for its $k$th power, where $k \in \mathbb{N}$, 
%\begin{equation}
%\begin{pmatrix} a & b \\ 1 & 0 \end{pmatrix}^{k} = \mathbf{U}  \begin{pmatrix} \lambda_{1}^{k}  & d \\ 0 & \lambda_{2}^{k} \end{pmatrix} \mathbf{U}^{T}
%\end{equation} 
%with some orthonormal matrix $\mathbf{U} \in \mathbb{R}^{2 \times 2}$ and $|d| \leq |a| +|b|+|c|   $.
\end{lemma} 
\begin{proof}
Let us introduce the shorthand $z \defeq \exp(j 2 \pi \theta )$ and further develop the 
LHS of \eqref{equ_lemma_identiy} as 
\begin{align} 
\label{long_integral_1}
& \int\limits_{\theta=0}^{1} \frac{z^{k\!-\!1}}{2(1\!-\!(z^{-1}\!+\!z)/2)\!+\!4/(\kappa\!-\!1)} d \theta \nonumber \\ 
% & = \int\limits_{\theta=0}^{1} \frac{z^{k\!-\!1}}{2(1\!-\!(z^{-1}\!+\!z)/2)\!+\!4/(\kappa\!-\!1)} d \theta  \nonumber \\ 
  & = \int\limits_{\theta=0}^{1} \frac{z^{k}}{2(z\!-\!(1\!+\!z^2)/2)\!+\!4z/(\kappa\!-\!1)} d \theta.  
\end{align}
The denominator of the integrand in \eqref{long_integral_1} can be factored as
\begin{equation}
\label{equ_facto_denom}
2(z\!-\!(1\!+\!z^2)/2)\!+\!4z/(\kappa\!-\!1) = -(z\!-\!z_{1})(z\!-\!z_{2}) 
\end{equation} 
with 
\begin{equation}
\label{equ_def_z_1_z_2}
z_{1} \defeq \frac{\sqrt{\kappa}\!+\!1}{\sqrt{\kappa}\!-\!1} \mbox{, and } z_{2} \defeq \frac{\sqrt{\kappa}\!-\!1}{\sqrt{\kappa}\!+\!1}.
\end{equation} 
Inserting \eqref{equ_facto_denom} into \eqref{long_integral_1}, 
\begin{align}
  &  \int\limits_{\theta=0}^{1} \frac{z^{k}}{2(z\!-\!(1\!+\!z^2)/2)\!+\!4/(\kappa\!-\!1)} d \theta  \nonumber \\ 
  & = - \int\limits_{\theta=0}^{1} \frac{z^{k}}{(z\!-\!z_{1})(z\!-\!z_{2})} d \theta  \nonumber \\ 
  & =  \int\limits_{\theta=0}^{1} -\frac{z^{k}(z_{1}\!-\!z_{2})^{-1}}{z\!-\!z_{1}}\!+\!\frac{z^{k}(z_{1}\!-\!z_{2})^{-1}}{z\!-\!z_{2}}  d \theta. \label{long_integral_2}
\end{align}
Since $|z_{2}| < 1$, we can develop the second term in \eqref{long_integral_2} by using the identity \cite[Sec. 2.7]{OppenheimSchaferBuck1998}
\begin{equation}
\label{equ_elem_ident_1}
\hspace*{-4mm}\int\limits_{\theta=0}^{1} \hspace*{-2mm}\frac{\exp(j 2 \pi k \theta)}{\exp(j2\pi \theta)\!-\!\alpha} d \theta\!=\!\alpha^{k-1}  \mbox{ for } k\!\in\!\mathbb{N}, \alpha\!\in\!\mathbb{R}, |\alpha|\!<\!1. 
\end{equation}
Since $|z_{1}| > 1$, we can develop the first term in \eqref{long_integral_2} by using the identity \cite[Sec. 2.7]{OppenheimSchaferBuck1998} 
\begin{equation}
\label{equ_elem_ident_2}
\hspace*{-4mm}\int\limits_{\theta=0}^{1} \hspace*{-2mm}\frac{\exp(j 2 \pi k \theta)}{\exp(j2\pi \theta)\!-\!\alpha} d \theta\!=\!0  \mbox{ for } k\!\in\!\mathbb{N}, \alpha\!\in\!\mathbb{R}, |\alpha|\!>\!1. 
\end{equation} 
Applying \eqref{equ_elem_ident_1} and \eqref{equ_elem_ident_2} to \eqref{long_integral_2}, 
\begin{equation}
\label{equ_long_int_111}
 \int\limits_{\theta=0}^{1} \frac{z^{k}}{2(z\!-\!(1\!+\!z^2)/2)\!+\!4/(\kappa\!-\!1)} d \theta\!=\! \frac{z_{2}^{k\!-\!1}}{z_{1}\!-\!z_{2}}. 
\end{equation}
Inserting \eqref{equ_long_int_111} into \eqref{long_integral_1}, we arrive at 
\begin{equation}
\label{equ_int_limits_theta_123}
\int\limits_{\theta=0}^{1}  \hspace*{-2mm}\frac{\exp(j2\pi (k\!-\!1) \theta)}{ 2(1\!-\!\cos(2\pi \theta))\!+\!4/(\kappa\!-\!1)}  d \theta = \frac{z_{2}^{k\!-\!1}}{z_{1}\!-\!z_{2}}. 
\end{equation}
The proof is finished by combining \eqref{equ_int_limits_theta_123} with the identity
\begin{equation}
\frac{1}{z_{1}\!-\!z_{2}} \stackrel{\eqref{equ_def_z_1_z_2}}{=} \frac{\sqrt{\kappa}\!+\!1}{\sqrt{\kappa}\!-\!1} - \frac{\sqrt{\kappa}\!-\!1}{\sqrt{\kappa}\!+\!1} = \frac{4 \sqrt{\kappa}}{\kappa\!-\!1}. \nonumber
\end{equation} 
\end{proof}

%\section{Acknowledgment} 

%\section{Proofs}
%\label{sec_proofs} 

\bibliographystyle{abbrv}
%\bibliography{Sampta2017}
%\bibliography{NNSP_JMLR}
\bibliography{CvxOptBib}

\end{document}